\documentclass{article}

\usepackage{microtype}
\usepackage{graphicx}
\usepackage{subcaption}
\usepackage{booktabs} 

\usepackage{hyperref}
\usepackage[english]{babel}
\usepackage{enumitem}

\usepackage[preprint]{icml2026}

\usepackage{amsmath}
\usepackage{amssymb}
\usepackage{mathtools}
\usepackage{amsthm}

\usepackage[utf8]{inputenc} 
\usepackage[T1]{fontenc}    
\usepackage{hyperref}       
\usepackage{url}            
\usepackage{booktabs}       
\usepackage{amsfonts}       
\usepackage{nicefrac}       
\usepackage{microtype}      
\usepackage{xcolor}         

\usepackage{empheq}
\usepackage{enumitem}
\usepackage{amsmath}
\usepackage{amssymb}
\usepackage{amsthm}
\usepackage{bm} 
\usepackage{algorithm}
\usepackage{algpseudocode}
\usepackage{float}

\newcommand{\rd}{\mathop{}\mathopen{}\textrm{d}}
\newcommand{\lfrac}[2]{(#1 \mathbin{/} #2)}

\usepackage{booktabs}
\usepackage{multirow}
\usepackage{makecell}
\usepackage[table]{xcolor}
\usepackage{threeparttable}
\usepackage{adjustbox}
\usepackage{caption}

\usepackage{graphicx}
\usepackage{subcaption}

\definecolor{headerblue}{RGB}{238,246,255}
\definecolor{oursblue}{RGB}{232,241,255}
\definecolor{bestgreen}{RGB}{222,245,226}
\definecolor{categorygray}{RGB}{245,245,245}

\usepackage{tabularx}
\definecolor{groupgray}{RGB}{246,247,249}
\definecolor{oursgreen}{RGB}{226,245,233}
\usepackage{wrapfig}
\usepackage[capitalize,noabbrev]{cleveref}

\theoremstyle{plain}
\newtheorem{theorem}{Theorem}[section]
\newtheorem{proposition}[theorem]{Proposition}

\theoremstyle{definition}

\theoremstyle{remark}

\usepackage[disable,textsize=tiny]{todonotes}

\icmltitlerunning{Physics-Informed Neural PDE Solvers via Spatio-Temporal MeanFlow}

\begin{document}

\twocolumn[
  \icmltitle{Physics-Informed Neural PDE Solvers via Spatio-Temporal MeanFlow}



  \icmlsetsymbol{equal}{*}

  \begin{icmlauthorlist}
    \icmlauthor{Hanru Bai}
     {yyy}
    \icmlauthor{Yuncheng Zhou}
    {yyy}
    \icmlauthor{Difan Zou}{comp}
  \end{icmlauthorlist}

  \icmlaffiliation{yyy}{Fudan University}
  \icmlaffiliation{comp}{The University of Hong Kong}

  \icmlcorrespondingauthor{Difan Zou}{dzou@cs.hku.hk}

  \icmlkeywords{Machine Learning, ICML}

  \vskip 0.3in
]



\printAffiliationsAndNotice{}  

\begin{abstract}
  Deep learning paradigms, such as PINNs and neural operators, have significantly advanced the solving of PDEs. However, they often struggle to capture the continuous integral nature of physical systems, relying either on pointwise residuals that ignore the integral perspective or on pre-discretized temporal grids. Drawing inspiration from MeanFlow, a continuous-time integrator recently developed to efficiently solve generative ODEs, we introduce Spatio-Temporal MeanFlow, which functions as a novel PDE solver learning the finite-interval evolution of physical states. By substituting the generative velocity field with the physical PDE operator, we transform multi-step numerical integration into an efficient prediction with a freely controllable integration length. Crucially, we extend the original MeanFlow constraint from the temporal to the spatio-temporal domain, coupling time evolution with spatial consistency. This yields a unified framework naturally accommodating both time-dependent and stationary PDEs. Comprehensive experiments on benchmarks demonstrate that our approach achieves superior accuracy and inference efficiency over representative baselines. Furthermore, the proposed integral constraint enables excellent generalization to out-of-distribution initial conditions and varying spatial resolutions.
\end{abstract}

\section{Introduction}

Solving partial differential equations (PDEs) efficiently and accurately is fundamental to modeling complex physical phenomena across science and engineering, such as shock wave propagation and fluid turbulence \cite{bandrauk2007high,madhavi2025advanced,glimm1991nonlinear}. Fundamentally, this mathematical process is a continuous integration process. For example, for a time-dependent PDE governed by $\frac{\partial \boldsymbol{u}}{\partial t} = \boldsymbol{f}[\boldsymbol{u}, \nabla, \boldsymbol{x}, t]$, the evolution of the state field $\boldsymbol{u}(\boldsymbol{x}, t)$ across the continuous spatial domain from time $\tau$ to $t$ is the integral:
\begin{equation}
 \textstyle \boldsymbol{u}(\boldsymbol{x}, t) = \boldsymbol{u}(\boldsymbol{x}, \tau) + \int_\tau^t \boldsymbol{f}[\boldsymbol{u}(\boldsymbol{x}, s), \nabla, \boldsymbol{x}, s]\,\mathrm{d}s.
\label{eq:integral_form}
\end{equation}            
 
While the two dominant deep learning paradigms for PDE solving, namely Physics-Informed Neural Networks (PINNs) \cite{cuomo2022scientific}, which embed governing equations as pointwise loss penalties, and neural operators like FNO \cite{li2020fourier}, which learn data-driven mappings between function spaces, have achieved remarkable success, they struggle to fully capture the continuous integral process defined in Eq. \eqref{eq:integral_form}. Specifically, FNO relies heavily on pre-discretized time slices and treats the integration interval as a fixed hyperparameter, essentially functioning as a discrete, step-by-step integrator. Conversely, PINNs rely solely on the pointwise PDE formula itself, bypassing the integral perspective entirely. Grounded in the formulation of Eq. \eqref{eq:integral_form}, we argue that a more straightforward approach is to directly train a neural integrator capable of operating flexibly in continuous domain. 

\begin{figure} 
    \centering
    \includegraphics[width=\linewidth]{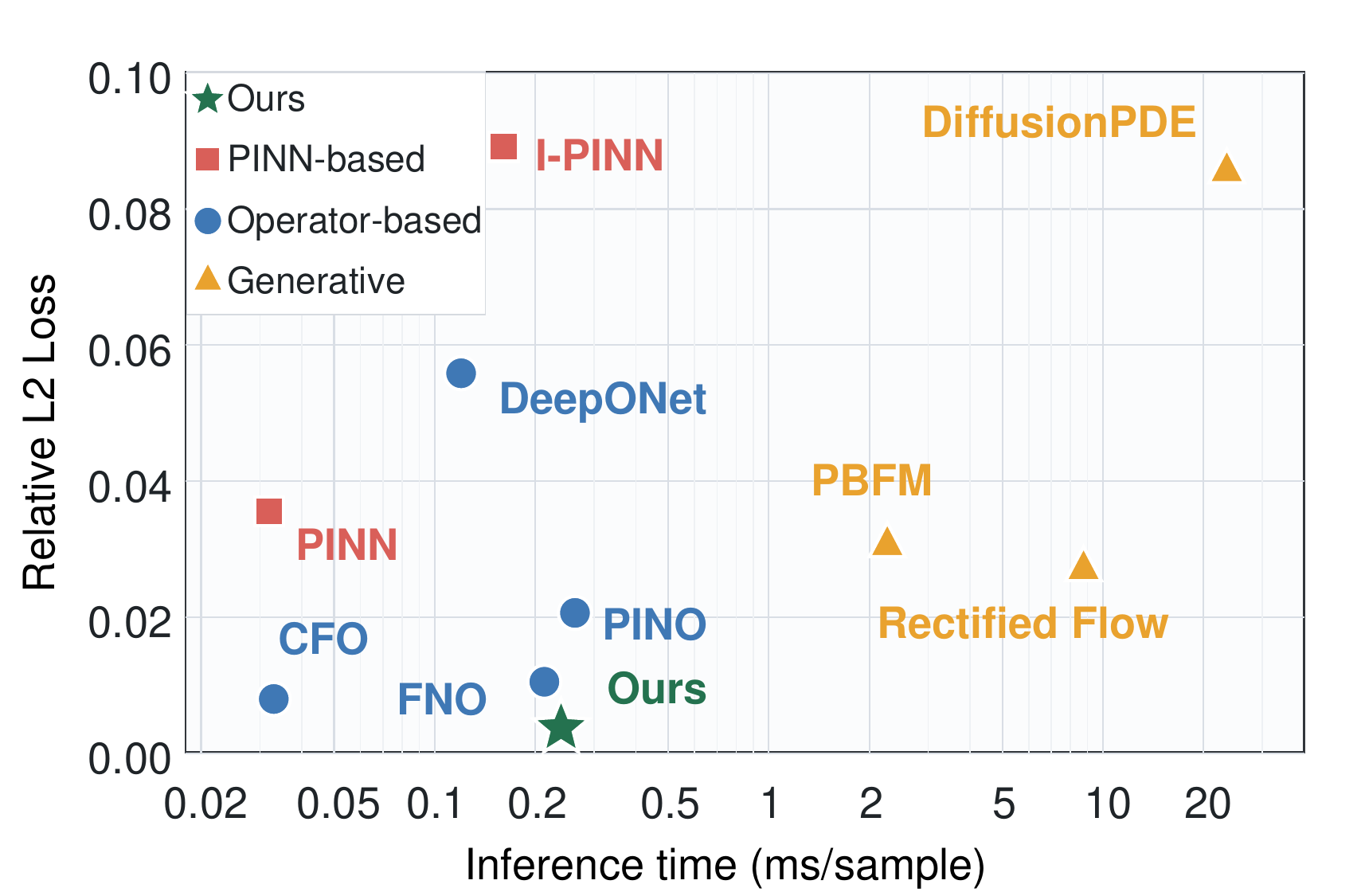}
    \caption{Accuracy-efficiency trade-off on the Burgers benchmark. Our Spatio-Temporal MeanFlow achieves an optimal Pareto frontier (bottom-left), delivering the lowest relative $L_2$ error with highly competitive inference speeds across all baselines.}
    \label{fig:teaser}
\end{figure}

A similar challenge has been extensively studied in generative modeling, such as diffusion models \cite{croitoru2023diffusion,yang2023diffusion,luo2022understanding} and flow matching \cite{lipman2022flow}, where the inference process can also be formulated as an ordinary differential equation (ODE). Consequently, several works strive to develop efficient methods to solve this generative ODE \cite{liu2022flow,song2023consistency}. One notable model is MeanFlow \cite{geng2025mean}, which directly learns the average velocity over a finite interval using neural networks. By doing so, it essentially functions as a continuous-time integrator whose core advantage lies in bypassing repeated velocity evaluations, transforming multi-step integration into an efficient, single-step prediction. This provides a powerful intuition: can we adapt this analogous mechanism to construct a continuous-time neural integrator dedicated to serving PDE solving?

Inspired by this, we adapt the continuous integral formulation of MeanFlow to PDE
solving. Specifically, we replace the generative velocity field
$\boldsymbol{v}$ in the MeanFlow identity with the PDE right-hand side
$\boldsymbol{f}$, and use the resulting identity to parameterize the
finite-interval evolution of PDE solutions. The original MeanFlow constraint
characterizes temporal evolution alone. For PDEs, however, generalization also
depends critically on spatial consistency. We therefore extend the MeanFlow
loss from the temporal domain to the spatio-temporal domain, introducing a
spatio-temporal MeanFlow loss tailored to PDE solving that strengthens the model's
ability to generalize across spatial locations and resolutions. Furthermore,
this formulation naturally accommodates static PDEs by specializing to a
purely spatial MeanFlow loss, providing a unified framework for both
time-dependent and stationary PDEs. By transcending purely
data-driven learning, the model achieves superior generalization to
out-of-distribution initial conditions and varied spatial resolutions.
Our contributions are:
\begin{itemize}
\item \textbf{A MeanFlow-inspired PDE solver framework}: We develop a novel PDE solver that enables efficient inference by transforming conventional multi-step
numerical integration into finite-interval predictions,
with the integration length freely controllable at inference time.

\item \textbf{A novel spatio-temporal MeanFlow loss}: We introduce a
physics-informed constraint that spatio-temporal significantly enhances solving accuracy
and generalization under out-of-distribution initial conditions and different spatial resolutions.

\item \textbf{Comprehensive experimental validation}: We conduct extensive evaluations on canonical PDE benchmarks and a real-world dataset. As shown in Fig.~\ref{fig:teaser}, our method achieves an optimal accuracy-efficiency trade-off across synthetic and real-world PDEs. Furthermore, extensive ablations validate our core components, demonstrating the integration-step flexibility of our continuous formulation and showing how the spatio-temporal MeanFlow losses drive robust generalization to higher resolutions and out-of-distribution (OOD) conditions.
\end{itemize}
\section{Preliminaries} 
\paragraph{MeanFlow.}
For states $\boldsymbol u_t$ ($\boldsymbol z_t$ in \cite{geng2025mean}), MeanFlow \cite{geng2025mean} achieves one-step prediction by modeling the average velocity $\boldsymbol{m}$ ($\boldsymbol u$ in \cite{geng2025mean}) along ODE $\frac{\rd\boldsymbol{u}_t}{\rd t} = \boldsymbol{v}(\boldsymbol{u}_t, t)$:
$
\textstyle\boldsymbol{m}(\boldsymbol{u}_t, {\tau}, t) \triangleq \frac{1}{t-{\tau}} \int_{\tau}^t \boldsymbol{v}(\boldsymbol{u}_s, s) \rd s
$ between time ${\tau}$ ($r$ in \cite{geng2025mean}) and $t$, 
rather than multi-step integration in Flow Matching. 
Then, by taking the total derivative with respect to time $t$ on both sides, we obtain the \textbf{MeanFlow Identity} in Eq.~\ref{eq:identity}, whose right-hand side serves as the training objective for network $\boldsymbol{m}_\theta(\boldsymbol{u}_t, {\tau}, t)$. (derived in App. \ref{app:meanflow_derivation})
\begin{equation}
\label{eq:identity}
\textstyle\boldsymbol{m}(\boldsymbol{u}_t, {\tau}, t) = \boldsymbol{v}(\boldsymbol{u}_t, t) - (t-{\tau})\frac{\rd}{\rd t}\boldsymbol{m}(\boldsymbol{u}_t, {\tau}, t).
\end{equation}

\paragraph{Problem Formulation.}
We formulate the PDE system over a spatial domain $D \subset \mathbb{R}^n$ and time $t \in [0, T]$ as follows:
\begin{flalign}
& \textstyle\gamma \frac{\partial \boldsymbol{u}(\boldsymbol{x}, t)}{\partial t}
= \boldsymbol{f}[\boldsymbol{u}(\boldsymbol{x}, t), \nabla, \boldsymbol{x}, t, \boldsymbol{a}],
\textstyle \boldsymbol{x} \in D, \, t \in (0, T], &&
\label{eq:pde_gov} \\
& \boldsymbol{u}(\boldsymbol{x}, 0)
= \boldsymbol{u}_0(\boldsymbol{x}),
 \boldsymbol{x} \in D, &&
\label{eq:pde_ic} \\
& \mathcal{B}[\boldsymbol{u}(\boldsymbol{x}, t)]
= \boldsymbol{g}(\boldsymbol{x}, t),
 \boldsymbol{x} \in \partial D, \, t \in [0, T]. &&
\label{eq:pde_bc}
\end{flalign}

Here, $\boldsymbol{u} \in \mathbb{R}^d$ is the state variable, $\boldsymbol{f}$ is the differential operator with parameters $\boldsymbol{a}$, while $\boldsymbol{u}_0$ and $\mathcal{B}$ denote the initial condition and boundary operator, respectively. The coefficient ${\gamma} \in \mathbb{R}$ unifies time-variant ($\gamma \neq 0$) and time-invariant ($\gamma = 0$) PDEs. Notably, for static systems ($\gamma=0$), the time dimension $t$ acts as a dummy variable, yielding a constant state $\boldsymbol{u}(\boldsymbol{x}, t) \equiv \boldsymbol{u}_{\text{static}}(\boldsymbol{x})$ along the temporal axis, which naturally accommodates our unified spatio-temporal modeling.

\begin{figure*}
    \centering
\includegraphics[width=0.85\linewidth]{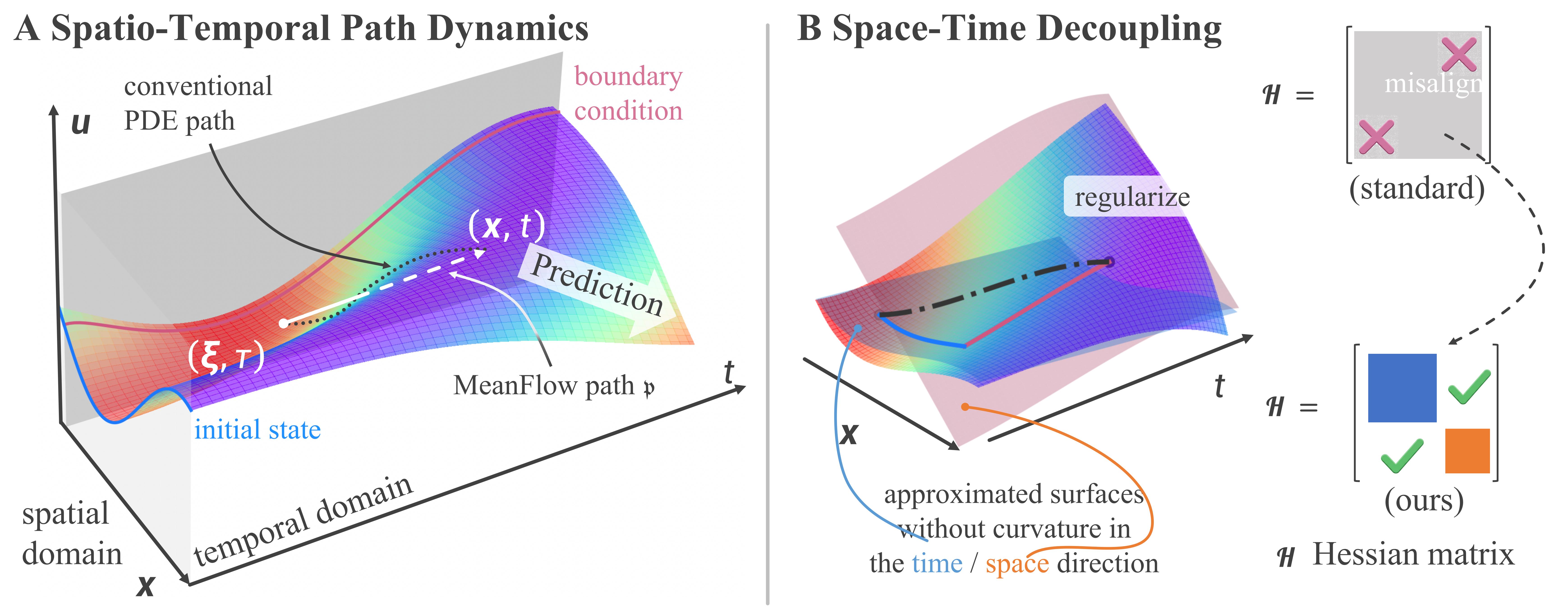} 
    \caption{Overview of our framework. (\textbf{A}) Spatio-Temporal Path Dynamics, formulating PDE solutions via straight paths. (\textbf{B}) Space-Time Decoupling, constraining the surface in space or time direction only, designed to mitigate scale mismatches within the Hessian matrix.}
    \label{fig:framework}
\end{figure*}
\section{Methodology}

This section presents the  Spatio-Temporal MeanFlow framework. Sec. \ref{sec:formulation} and \ref{sec:constraints} formulate the problem via integration paths and derive physics-informed governing constraints. To mitigate temporal-spatial scale mismatches, Sec. \ref{sec:decoupling} introduces a decoupling strategy with a theoretical decoupling bound. Sec. \ref{sec:framework} integrates these into a unified objective function and details the specific algorithms. Finally, Sec. \ref{sec:theory} establishes theoretical foundations, providing rigorous justifications for consistency, a-posteriori error bounds, and improved optimization dynamics.
\subsection{Spatio-Temporal MeanFlow Formulation}
\label{sec:formulation}
We consider the general partial differential equation (PDE) system defined by Eq. \ref{eq:pde_gov}.
To relate the states between a starting point $(\boldsymbol{\xi}, \tau)$ and a target point $(\boldsymbol{x}, t)$, we define a spatio-temporal flow path $\mathfrak{p}: [0,1] \mapsto D \times [0,T]$ such that $\mathfrak{p}(0)=(\boldsymbol{\xi}, \tau)$ and $\mathfrak{p}(1)=(\boldsymbol{x}, t)$ (Fig. \ref{fig:framework} A). 
We then define the spatio-temporal MeanFlow function $\boldsymbol{m}$ as the integral average of the state evolution along path $\mathfrak{p}$:
\begin{equation}
\textstyle\boldsymbol{m}[\boldsymbol{u}(\boldsymbol{\xi}, \tau), \mathfrak{p}, \boldsymbol{a}] \triangleq {\int_{\mathfrak{p}} \rd\boldsymbol{u}(\mathfrak{p})}~/~{\int_{\mathfrak{p}} \|\rd \mathfrak{p}\|_2}.
\label{eq:meanflow_identity}
\end{equation}

For a convex domain $D$, we consider a straight path $\mathfrak p(s)\triangleq[\boldsymbol \xi  + s(\boldsymbol x-\boldsymbol \xi ), \tau + s(t-\tau )]$ and define path length $l(\boldsymbol \xi , \tau , \boldsymbol x, t)\triangleq\int_{\mathfrak p}\Vert \textrm d\mathfrak p\Vert _2=\sqrt{\Vert \boldsymbol x-\boldsymbol \xi \Vert _2^2+(t-\tau )^2}$ which is the total spatio-temporal Euclidean distance. 
Following the Fundamental Theorem of Line Integrals, Eq. \ref{eq:meanflow_identity} satisfies
\begin{equation}
\textstyle l(\boldsymbol \xi , \tau , \boldsymbol x, t)\cdot\boldsymbol m[\boldsymbol u(\boldsymbol \xi , \tau ), \boldsymbol \xi , \tau , \boldsymbol x, t; \boldsymbol a] = \boldsymbol u(\boldsymbol x, t) - \boldsymbol u(\boldsymbol \xi , \tau ). 
\label{eq:fund-id}
\end{equation}
Note that using a straight path requires $D$ to be convex in classical analysis, as it might temporarily exit the physical boundary of non-convex domains (e.g., fluid flow around obstacles). Fortunately, parameterized neural networks act as global continuous functions, hence implicitly perform a \textit{smooth analytic continuation} of the physical state field into the convex hull of $D$. 

In our formulation, the proposed path integration and differentiation operators rely on two fundamental mathematical prerequisites: 1) The state variable $\boldsymbol u$ is sufficiently smooth. Specifically, we require $\boldsymbol{u} \in \mathcal{C}^1(D \times [0,T])$ for first-order PDEs, and $\boldsymbol{u} \in \mathcal{C}^2$ for systems involving second-order spatial derivatives. This guarantees that the Jacobian-based operators (e.g., $\boldsymbol K$ in Sec. \ref{sec:constraints}) are well-defined; 2) The domain is topologically connected to ensure the physical meaningfulness of bridging any two states. 

\subsection{Physics-informed MeanFlow Constraints}
\label{sec:constraints}
By differentiating the fundamental identity Eq. \eqref{eq:fund-id} with respect to the initial coordinates $(\boldsymbol{\xi}, \tau)$, we derive the physics-informed constraints for our framework, i.e., 
\begin{align}
& \textstyle l_t \cdot \boldsymbol m
= l \boldsymbol K \boldsymbol f
+ \gamma l^2 \cdot \frac{\partial \boldsymbol m}{\partial \tau},
\label{eq:time-mf}
\\
& \left\{
\begin{aligned}
\textstyle l \cdot \boldsymbol K \cdot \nabla^\top \boldsymbol u(\boldsymbol \xi, \tau)
&= \boldsymbol m(\boldsymbol x - \boldsymbol \xi)^\top
- l^2 \cdot \frac{\partial \boldsymbol m}{\partial \boldsymbol \xi},
\\
\textstyle l \cdot \boldsymbol K \cdot \Delta \boldsymbol u(\boldsymbol \xi, \tau)
&= -(\boldsymbol \alpha + 2 \boldsymbol K^{-1} \boldsymbol \beta),
\quad \cdots .
\end{aligned}
\right.
\label{eq:spatial-mf}
\end{align}
Here, $\boldsymbol{\alpha}\triangleq n\boldsymbol{m}-\lfrac{l_s^2}{l^2}\cdot\boldsymbol{m}+2l^2{\Delta}\boldsymbol{m}(\boldsymbol{\xi}, {\tau})$ and $\boldsymbol{\beta}\triangleq-\frac{\partial\boldsymbol{m}}{\partial \boldsymbol{\xi}}(\boldsymbol{x}-\boldsymbol{\xi})-\lfrac{l_s^2}{l}\cdot\frac{\partial\boldsymbol{m}}{\partial \boldsymbol{u}}\boldsymbol{m}$, with $l_t = t - \tau$ being the time step, $l_s = \|\boldsymbol{x}-\boldsymbol{\xi}\|_2$ the spatial distance, and $\boldsymbol{K} = \boldsymbol{I} + l \frac{\partial \boldsymbol{m}}{\partial \boldsymbol{u}}$ the Jacobian-based correction operator. Specifically, Eq. \ref{eq:time-mf} is derived by differentiating Eq.~\eqref{eq:fund-id} with respect to time $\tau$, while Eq. \ref{eq:spatial-mf} are obtained by taking spatial derivatives (first-order, second-order, and so forth) with respect to $\boldsymbol{\xi}$. Prop. \ref{app:prop:mf-constraints}, Appendix \ref{app:theory} provides the detailed derivation. Consequently, the spatio-temporal MeanFlow loss is defined as
\begin{equation}
\mathcal L_\textrm{ST-MeanFlow} \triangleq \Vert \textstyle l_t\cdot\boldsymbol m - l\boldsymbol K\boldsymbol f - {\gamma}l^2\cdot\frac{\partial\boldsymbol m}{\partial {\tau}}\Vert _2^2, 
\label{eq:mf-loss}
\end{equation}
with the gradient terms in $\boldsymbol f$ substituted by Eqs. \eqref{eq:spatial-mf}. 

\subsection{Space-Time Decoupling Strategy}
\label{sec:decoupling}

In many physical systems, temporal evolution and spatial variations occur at vastly different magnitudes (e.g., fast chemical reactions in large-scale reactors) \cite{wang2021understanding,krishnapriyan2021characterizing}. When these dimensions are tightly coupled in a single loss function, the optimization landscape often exhibits pathological curvature, characterized by a highly ill-conditioned Hessian matrix \cite{wang2022and} with large off-diagonal cross-coupling terms (Fig. \ref{fig:framework} B). To solve the challenge of inherent \textit{spatio-temporal scale mismatch}, we decompose the spatio-temporal objective into two independent first-order components:

\textbf{Temporal MeanFlow Loss.} $\mathcal{L}_{\textrm{T-MeanFlow}} \triangleq \left\| \boldsymbol{m} - \boldsymbol K\boldsymbol f - {\gamma}l_t \frac{\partial \boldsymbol{m}}{\partial \tau} \right\|_2^2$ constrains the network to satisfy the temporal evolution predicted by the Spatio-Temporal MeanFlow identity along the time-line. 

\textbf{Spatial MeanFlow Loss.} $\mathcal{L}_{\textrm{S-MeanFlow}} \triangleq \left\| l_s \boldsymbol K \nabla^\top \boldsymbol{u} - \boldsymbol{m}(\boldsymbol{x}-\boldsymbol{\xi})^\top + l_s^2 \frac{\partial \boldsymbol{m}}{\partial \boldsymbol{\xi}} \right\|_2^2$ enforces consistency between predicted Spatio-Temporal MeanFlow's spatial derivatives and numerical spatial gradients. 

Thm. \ref{thm:decoupling-error}, Sec. \ref{sec:theory} ensures that the decoupling provide a theoretical reference for the overall Spatio-Temporal MeanFlow loss. 
The core advantage of this decoupling is the removal of explicit spatio-temporal cross-derivatives in the residual definitions. Mathematically, this transformation reshapes the Hessian into a nearly block-diagonal structure, reducing its condition number and enabling faster, more robust convergence (Thm. \ref{app:thm:hessian}, App. \ref{app:hessian-condition}). This decoupling is theoretically supported by the \textit{Space-Time Decoupling Bound} presented in the following section, which ensures that minimizing these independent losses effectively constrains the global spatio-temporal error.
    
        \begin{algorithm}[t] 
        \small
        \caption{Training}
        \label{alg:training}
        \begin{algorithmic}[1]
        \Require PDE params $\boldsymbol{a}$, temporal indicator ${\gamma}$, network $\boldsymbol{m}_\theta$, learning rate $\eta$
        \For{each training iteration}
            \If{${\gamma} \neq 0$} \Comment{Time-dependent PDEs}
                \State Sample time pairs $\tau < t$ and grid $\boldsymbol{x}$;
                \State Get input: $\boldsymbol{I} \gets [\boldsymbol{u}(\boldsymbol{x}, \tau), \tau, \boldsymbol{x}, t, \boldsymbol{a}]$;
                \State Predict MeanFlow: $\boldsymbol{m} \gets \boldsymbol{m}_\theta(\boldsymbol{I})$;
                \State Compute 
                $\boldsymbol{u}(\boldsymbol{x}, t) \gets \boldsymbol{u}(\boldsymbol{x}, \tau) + (t-\tau)  \boldsymbol{m}$;
                \State Compute Loss: $\mathcal{L}_{\textrm{Total}} 
                $;
            \Else \Comment{Static PDEs (${\gamma} = 0$)}
                \State Analytically extend boundary conditions 
                \Statex \quad\quad to the entire domain as $\boldsymbol{u}_0(\boldsymbol{x})$; 
                \State Get input: $\boldsymbol{I} \gets [\boldsymbol{u}_0(\boldsymbol{x}), \boldsymbol{a}]$;
                \State Predict MeanFlow: $\boldsymbol{m} \gets \boldsymbol{m}_\theta(\boldsymbol{I})$;
                \State Construct state: $\boldsymbol{u}(\boldsymbol{x}) \gets \boldsymbol{u}_0(\boldsymbol{x}) + \boldsymbol{m}$;
                \State Compute Loss: $\mathcal{L}_{\textrm{Total}} \gets \mathcal{L}_{\textrm{Data}} + \lambda_{s}\mathcal{L}_{\textrm{S-MF}}$; 
            \EndIf
            \State Update parameters: $\theta \gets \theta - \eta \nabla_\theta \mathcal{L}_{\textrm{Total}}$;
        \EndFor
        \State \Return $\theta^*={\theta}$;
        \end{algorithmic}
        \end{algorithm}
    
\subsection{Detailed Implemention}
\label{sec:framework}

\paragraph{Training.} 
We parameterize the MeanFlow predictor $\boldsymbol{m}$ using a neural network, denoted as $\boldsymbol{m}_\theta$, to approximate the true $\boldsymbol{m}$. For time-dependent PDEs, the network is designed to map the concatenated input $[\boldsymbol{u}(\boldsymbol{x}, \tau), \tau, \boldsymbol{x}, t, \boldsymbol{a}]$, where $\tau$ and $t$ are randomly sampled time pairs during training. The model outputs the average integral from time $\tau$ to $t$. According to the fundamental identity (Eq.~\eqref{eq:fund-id}), the target state at time $t$ is explicitly constructed as $\boldsymbol{u}(\boldsymbol{x}, t) = \boldsymbol{u}(\boldsymbol{x}, \tau) + (t-\tau) \cdot \boldsymbol{m}_\theta$. 
The total objective function is defined as the weighted sum of the data loss (measuring the discrepancy between the predicted solution at time $t$ and observational data) and the physics-informed decoupled spatio-temporal MeanFlow losses:
\begin{equation}
\textstyle\mathcal{L}_{\textrm{Total}} =  \mathcal{L}_{\textrm{Data}} + \lambda_{t}\mathcal{L}_{\textrm{T-MeanFlow}} + \lambda_{s}\mathcal{L}_{\textrm{S-MeanFlow}},
\label{eq:loss_dynamic}
\end{equation}
where $\lambda_t$ and $\lambda_s$ are the weighting coefficients.

Conversely, when $\gamma = 0$, the system degenerates into a static PDE. In this scenario, temporal variables are omitted. The model instead takes the analytically extended boundary conditions $\boldsymbol{u}_0(\boldsymbol{x})$ (where $\boldsymbol{u}_0\rvert_{\partial D} = \boldsymbol{u}\rvert_{\partial D}$) and the coefficient field $\boldsymbol{a}$ as inputs. The solution over the interior domain $\boldsymbol{x} \in D \setminus \partial D$ is constructed as $\boldsymbol{u}(\boldsymbol{x}) = \boldsymbol{u}_0(\boldsymbol{x}) + \boldsymbol{m}_\theta$, where $\boldsymbol{m}_\theta$ represents the network output. The temporal loss naturally vanishes, and the training objective simplifies to:
$
    \mathcal{L}_{\textrm{Total}} =  \mathcal{L}_{\textrm{Data}} + \lambda_{s}\mathcal{L}_{\textrm{S-MeanFlow}}.
$
Alg. \ref{alg:training} detail the training optimization procedures of the proposed framework. 

Note that in our practical implementation, considering that the benchmark PDEs evaluated in this work are defined on regular grids, we maintain fixed spatial coordinates ($\boldsymbol{x}$) across inputs and outputs. Consequently, our current network design predominantly emphasizes continuous-time evolution. However, our fundamental mathematical formulation inherently supports continuous spatial evolution as well, where the anchor point $\boldsymbol{\xi}$ and the target point $\boldsymbol{x}$ can vary continuously. The detailed algorithms for the fully generalized spatio-temporal framework are provided in App.~\ref{app:generalized_algorithm}.
\begin{algorithm}[t] 
        \small
        \caption{Multi-step Inference}
        \label{alg:inference}
        \begin{algorithmic}[1]
        \Require Pretrained network $\boldsymbol{m}_{\theta^{*}}$, initial state $\boldsymbol{u}_0=\boldsymbol u(\boldsymbol x, 0)$, PDE params $\boldsymbol{a}$, target time $T$, steps $N$
        \State Calculate time step: $\Delta t \gets T / N$;
        \State Initialize trajectory list with $\boldsymbol{u}_0$;
        \State Initialize current time: $t_0 \gets 0$;
        \For{$k = 0$ to $N-1$}
            \State Target time: $t_{k+1} \gets t_k + \Delta t$;
            \State Get input: 
            \Statex \quad\quad $\boldsymbol{I}_k \gets [\boldsymbol{u}_k, t_k, \boldsymbol{x}, t_{k+1}, \boldsymbol{a}]$;
            \State Predict average integral (MeanFlow): 
            \Statex \quad\quad $\boldsymbol{m}_k \gets \boldsymbol{m}_{\theta^{*}}(\boldsymbol{I}_k)$;
            \State Update state via continuous integral: 
            \Statex \quad\quad $\boldsymbol{u}_{k+1} \gets \boldsymbol{u}_k + \Delta t \cdot \boldsymbol{m}_k$;
            \State Append $\boldsymbol{u}_{k+1}$ to trajectory list;
        \EndFor
        \State \Return $\boldsymbol{u}_T$;
        \end{algorithmic}
        \end{algorithm}

\begin{table*}[t]
\centering
\small
\setlength{\tabcolsep}{4.8pt}
\renewcommand{\arraystretch}{1.15}
\caption{Results for forward PDE solving on four synthetic benchmarks and one real-world dataset, alongside the inverse problem on Darcy flow. We report the relative $L_2$ error ($\downarrow$) as the mean $\pm$ standard deviation across three different random seeds. All reported metrics are multiplied by $10^{-2}$ for readability. Results for CFO on steady-state PDEs are omitted, as it is designed for time-dependent system. The best results are highlighted in \textbf{bold}, while the second-best results are \underline{underlined}.}
\label{tab:main_results}
\begin{tabular*}{\textwidth}{@{} l @{\hspace{12pt}}@{\extracolsep{\fill}} c c c c c c @{}}
\toprule
\multirow{2}{*}{\makecell[l]{Method \\ ($\times 10^{-2}$)}} &
\multicolumn{4}{c}{Synthetic Forward PDEs} &
\multicolumn{1}{c}{Real-world} &
\multicolumn{1}{c}{Inverse} \\
\cmidrule(lr){2-5} \cmidrule(lr){6-6} \cmidrule(l){7-7}
& Burgers & Navier--Stokes & Darcy & Poisson & HYCOM & Darcy \\
\midrule

\rowcolor{categorygray}[0pt][0pt]
\multicolumn{7}{@{}l}{\textit{Physics-informed coordinate-based solvers}} \\
PINN     & 3.56$\pm$0.42 & 40.21$\pm$0.20 & 15.98$\pm$1.55 & 32.01$\pm$4.69 & 28.12$\pm$0.24 & 27.41$\pm$0.01 \\
I-PINN   & 8.92$\pm$3.19 & 35.65$\pm$0.15 & 8.41$\pm$0.12 & 54.64$\pm$4.09 & 33.05$\pm$0.45 & 8.95$\pm$0.01 \\

\midrule
\rowcolor{categorygray}[0pt][0pt]
\multicolumn{7}{@{}l}{\textit{Neural operator solvers}} \\
DeepONet & 5.59$\pm$0.07 & 19.68$\pm$1.14 & 2.92$\pm$0.02 & 5.46$\pm$0.33 & 34.23$\pm$0.54 & 8.73$\pm$0.26 \\
FNO      & 1.05$\pm$0.07 & 0.38$\pm$0.00 & \underline{0.46$\pm$0.00} & \underline{0.23$\pm$0.01} & 24.08$\pm$0.79 & \underline{1.40$\pm$0.00} \\
PINO     & 2.06$\pm$0.18 & \underline{0.32$\pm$0.00} & 16.30$\pm$0.03 & 1.58$\pm$0.00 & \underline{23.79$\pm$0.11} & 12.19$\pm$0.32 \\
CFO      & \underline{0.80$\pm$0.15} & 0.88$\pm$0.00 & -- & -- & 32.75$\pm$4.26 & 17.07$\pm$1.01 \\

\midrule
\rowcolor{categorygray}[0pt][0pt]
\multicolumn{7}{@{}l}{\textit{Generative PDE solvers}} \\
DiffusionPDE    & 8.61$\pm$0.31 & 118.74$\pm$0.95 & 5.48$\pm$1.75 & 0.76$\pm$0.03 & 180.89$\pm$0.72 & 4.49$\pm$0.94 \\
Rectified Flow  & 2.76$\pm$0.06 & 1.40$\pm$0.28 & 0.71$\pm$0.07 & 1.58$\pm$0.00 & 31.80$\pm$0.21 & 36.95$\pm$1.88 \\
PBFM            & 3.11$\pm$0.60 & 137.45$\pm$15.35 & 0.62$\pm$0.01 & 28.56$\pm$0.07 & 83.86$\pm$5.43 & \bf{1.02$\pm$0.03} \\

\midrule
\textbf{Ours} & \bf{0.26$\pm$0.00} & \bf{0.19$\pm$0.02} & \bf{0.25$\pm$0.00} & \bf{0.18$\pm$0.02} & \bf{23.43$\pm$0.12} & 2.72$\pm$0.00 \\
\bottomrule
\end{tabular*}
\end{table*}
\paragraph{Inference.}
We first focus on the one-step inference scenario. For time-dependent PDEs, standard deep learning-based solvers typically take the initial condition and physical coefficients as inputs to predict the solution over the spatial domain at a target time $t$. Following this paradigm, our algorithm takes $[\boldsymbol{u}(\boldsymbol{x}, 0), 0, \boldsymbol{x}, t, \boldsymbol{a}]$ as input during inference to predict the average integral from time $0$ to $t$, denoted as $\boldsymbol{m}_{\boldsymbol{\theta}}(\boldsymbol{x},t)$. The final solution at time $t$ is then recovered via 
\begin{equation}
\textstyle\boldsymbol{u}(\boldsymbol{x}, t) = \boldsymbol{u}(\boldsymbol{x}, 0) + t \cdot \boldsymbol{m}_\theta(\boldsymbol{x}, t).
\end{equation}
For static systems, the general objective is to derive the global solution given specific boundary conditions and coefficient fields. Maintaining consistency with our training phase, the model simply takes the analytically extended boundary conditions and the coefficient field as inputs and then get the solution across the spatial domain through $\boldsymbol{u}(\boldsymbol{x}) = \boldsymbol{u}_0(\boldsymbol{x}) + \boldsymbol{m}_\theta$. Furthermore, for time-dependent PDEs, since our framework models the continuous integral over arbitrary time intervals, it naturally extends to multi-step inference. The detailed procedure for multi-step prediction is provided in Alg.~\ref{alg:inference}.

\subsection{Theoretical Results}
\label{sec:theory}
The proposed framework is supported by two core theoretical pillars: a-posteriori error bounds for the predicted state (Thm. \ref{thm:error_bound}), and effectiveness of the decoupled loss terms (Thm. \ref{thm:decoupling-error}). App. \ref{app:error-estimation}-\ref{app:hessian-condition} provide proofs for them, with the improved optimization landscapes through decoupling (Thm. \ref{app:thm:hessian}). 

For simplicity, we define $\boldsymbol r_\textrm{MeanFlow}\triangleq l_t\boldsymbol m - l\boldsymbol K\boldsymbol f-{\gamma}l^2\frac{\partial\boldsymbol m}{\partial {\tau}}$ with higher-order spatial gradients in $\boldsymbol h$ substituted by spatial MeanFlow Eqs. \eqref{eq:spatial-mf}, $\boldsymbol r_\textrm{Temp}\triangleq\boldsymbol m - \boldsymbol K\boldsymbol f-{\gamma}l_t\frac{\partial\boldsymbol m}{\partial {\tau}}$, and $\boldsymbol r_\textrm{Spac}\triangleq l_s\boldsymbol K{\nabla}^\top\boldsymbol u - \boldsymbol m(\boldsymbol x-\boldsymbol {\xi})^\top+l_s^2\frac{\partial\boldsymbol m}{\partial \boldsymbol {\xi}}$ as the residual vectors for different MeanFlow losses. 



\begin{theorem}[A-posteriori Error Estimate]
\label{thm:error_bound}
Assume the MeanFlow predictor $\boldsymbol{m}_\theta$ is $\mathcal{C}^1$-continuous. If the total training loss bounds the squared path residuals such that $\mathcal{L}_{\textrm{MeanFlow}} \le \epsilon$, then the pointwise error between reconstructed state $\boldsymbol{u}_\theta$ and true solution $\boldsymbol{u}_{\textrm{True}}$ at any target point $(\boldsymbol{x}, t)$ is bounded by:
\begin{equation}
    \|\boldsymbol{u}_\theta(\boldsymbol{x}, t) - \boldsymbol{u}_{\textrm{True}}(\boldsymbol{x}, t)\|_2 \le C \sqrt{\epsilon} + \mathcal{O}(l^2),
\end{equation}
where $C$ is a constant depending on the path length $l$ and the Lipschitz constant of the PDE operator.
\end{theorem}

Thm. \ref{thm:error_bound} provides a bridge between the optimization target and the physical accuracy, justifying the use of MeanFlow residuals as a proxy for solution quality.

\begin{theorem}[Space-Time Decoupling Bound]
\label{thm:decoupling-error}
Assume that the physical system satisfies regular linear gradient structure, whose differential function $\boldsymbol f$ can be split into ${\nabla}^\top\boldsymbol u\cdot\boldsymbol g + \boldsymbol h$ where $\boldsymbol h$ is irrelevant to first-order gradient ${\nabla}\boldsymbol u$. The global MeanFlow loss is bounded by space-time decoupled losses:
\begin{equation}
    \mathcal{L}_{\textrm{MeanFlow}} \le c_1 \mathcal{L}_{\textrm{Temp}} + c_2 \mathcal{L}_{\textrm{Spac}} + {\varepsilon}
\end{equation}
where the coefficients are defined as $c_1 = 3l_t^2$, and $c_2 = 3 \left( \frac{l - l_t}{l_s} \|\boldsymbol{g}\|_{\textrm{op}} \right)^2$. Here, ${\varepsilon}$ represents the structural coupling error and ${\varepsilon}=\mathcal O(l_s^2)$.
\end{theorem}

The theoretical validity of decoupled training is established in Thm. \ref{thm:decoupling-error}, which proves that the global spatio-temporal consistency is bounded by the individual temporal and spatial errors.
This bound ensures that minimizing decoupled losses effectively controls the total physics-informed error. In addition, the gap between decoupled losses and spatio-temporal consistency is controlled by $l_s$, which can theoretically be infinitesimally small due to continuous modeling.


\begin{table*}[t]
\centering
\small
\setlength{\tabcolsep}{4pt}
\renewcommand{\arraystretch}{1.15}
\caption{Performance and computational cost on the Burgers benchmark. We report the relative $L_2$ error ($\times 10^{-2}$), inference time (milliseconds per sample), and GFLOPs per sample. Lower is better. The best results are highlighted with bold text and a gray background.}
\label{tab:efficiency_burgers}
\resizebox{\linewidth}{!}{
\begin{tabular}
{@{}lcccccccccc@{}}
\toprule
& \multicolumn{2}{c}{\textit{Coord-based}} & \multicolumn{4}{c}{\textit{Neural operators}} & \multicolumn{3}{c}{\textit{Generative}} & \\
\cmidrule(lr){2-3} \cmidrule(lr){4-7} \cmidrule(lr){8-10}
Metric & PINN & I-PINN & DeepONet & FNO & PINO & CFO & DiffPDE & RectFlow & PBFM & Ours \\
\midrule
Error ($\times 10^{-2}$) $\downarrow$ & 3.56 & 8.92 & 5.59 & 1.05 & 2.06 & 0.80 & 8.61 & 2.76 & 3.11 & \cellcolor{groupgray}\textbf{0.26} \\
\addlinespace[4pt]
Time (ms) $\downarrow$ & \cellcolor{groupgray}\textbf{0.032} & 0.161 & 0.120 & 0.213 & 0.263 & 0.033 & 23.487 & 8.753 & 2.263 & 0.239 \\
\addlinespace[4pt]
GFLOPs $\downarrow$ & 0.0715 & 0.0764 & \cellcolor{groupgray}\textbf{0.0004} & 0.0012 & 0.2742 & 0.0007 & 662.6990 & 155.3370 & 12.9219 & 0.0118 \\
\bottomrule
\end{tabular}
}
\end{table*}
\section{Experiments}
In this section, we comprehensively evaluate our Spatio-Temporal MeanFlow framework. First, in Sec.~\ref{subsec:forward}, we validate its PDE solving accuracy and computational efficiency against representative baselines on four PDE benchmarks. Next, Sec.~\ref{subsec:ablation} presents extensive ablations to verify the effectiveness of our core components. Specifically, we demonstrate the integration-step flexibility enabled by our framework functioning as a continuous integrator over arbitrary time intervals, and show how the temporal and spatial MeanFlow losses drive generalization to higher resolutions and OOD initial conditions. Finally, Sec.~\ref{subsec:broader_applicability} highlights the framework's broader applicability beyond synthetic forward prediction, showcasing its potential on real-world datasets and inverse problems.

\subsection{Forward PDE Solving on Synthetic Benchmarks}
\label{subsec:forward}

\paragraph{Benchmarks.}
We first evaluate the effectiveness of the proposed method for PDE solving using four standard synthetic benchmarks. The benchmarks include two time-dependent PDEs, 1D Burgers' equation and the 2D Navier--Stokes equation, and two steady-state PDEs, 2D Darcy flow and the 2D Poisson equation. All datasets are generated using numerical solvers and are divided into training, validation, and test sets with a ratio of 7:2:1. Further details about the governing equations, numerical discretizations, and data generation configurations are provided in App.~\ref{app:data}.

\paragraph{Baseline Methods and Metrics.}
We compare our method with three categories of PDE solvers: (1) physics-informed coordinate-based methods: PINN \cite{cuomo2022scientific}, I-PINN \cite{niu2025improved}; (2) neural operators: FNO \cite{li2020fourier}, DeepONet \cite{lu2019deeponet}, PINO \cite{li2024physics}, CFO \cite{hou2026cfo}); and (3) generative solvers: DiffusionPDE \cite{huang2024diffusionpde}, Rectified Flow \cite{armegioiu2025rectified}, PBFM \cite{pbfm2026}. All compared methods are reproduced using their publicly released code. Further details are provided in App.~\ref{app:baselines}. Following \cite{huang2024diffusionpde}, we report the relative $L_2$ error on the test set.
This metric enables fair cross-dataset comparisons regardless of solution magnitudes.

\paragraph{Implementation Details.}
All experiments are implemented in PyTorch and conducted on NVIDIA GeForce RTX 4080 GPUs. We employ the Adam \cite{kingma2014adam} optimizer for training. The learning rates are set to $5.5\times10^{-4}$, $2.0\times10^{-4}$, and $3.0\times10^{-4}$ for the Burgers' equation, NS equation, and steady-state PDEs. 
For the loss functions, the MeanFlow coefficients are consistently set to $\lambda_t = \lambda_s = 0.01$, with $\lambda_t$ omitted for steady-state cases. The model checkpoint achieving the lowest validation error is used for testing. Detailed hyperparameters and network architectures are provided in App.~\ref{app:implementation}.

\paragraph{Main Results.}
Table~\ref{tab:main_results} reports the forward prediction results on the four benchmarks. Our method achieves the best overall performance across both time-dependent and steady-state PDEs. On Burgers' equation and NS equation, the improvement indicates that directly learning finite-interval evolution can reduce temporal error accumulation compared with methods that rely only on supervised data fitting or pointwise physical residuals. On Darcy flow and Poisson equation, our method also obtains lower errors, showing that the spatial MeanFlow constraint provides an effective regularization. 
Figure~\ref{fig:ns_vis} shows qualitative results on NS test samples. Ours closely matches the ground-truth vorticity fields across different time steps, capturing the main flow structures and temporal evolution. Additional visual comparisons with baseline methods and training curves for all loss components (the total loss, the data loss, the temporal and the spatial MeanFlow loss) are detailed in App.~\ref{app:results}.

\paragraph{Efficiency and Computational Cost}
\label{app:efficiency}

Tab.~\ref{tab:efficiency_burgers} reports the detailed inference latency and computational cost on the Burgers benchmark. 
Inference time is measured in milliseconds per sample, and computational cost is measured in GFLOPs per sample. 
Our method achieves low prediction error with sub-millisecond inference and only $0.0118$ GFLOPs per sample. 
Compared to generative PDE solvers, it is substantially more efficient as it avoids the costly multi-step sampling or ODE evaluations. Meanwhile, when compared to similarly fast neural operator methods, our framework delivers lower prediction errors.

\begin{table*}[t]
\centering
\small
\setlength{\tabcolsep}{4pt}
\renewcommand{\arraystretch}{1.08}
\caption{Ablation study on Burgers evaluating the effects of temporal and spatial MeanFlow losses, resolution generalization and out-of-distribution initial condition generalization. All models are trained at the base resolution $s=128$ and evaluated on test sets generated with the same random seed for each setting. We report global relative $L_2$ error; lower is better.}
\label{tab:burgers_generalization_ablation}
\begin{tabularx}{\textwidth}{@{}lXccc>{\columncolor{gray!12}[0pt][0pt]}c@{}}
\toprule
Type & Test Setting & w/o Both & w/o T-MF & w/o S-MF & Ours \\
\midrule
\multirow{3}{*}{Resolution}
& $s=128$ & 0.017561 & 0.003075 & 0.010821 & 0.002574 \\
& $s=256$ & 0.017562 & 0.003051 & 0.010820 & 0.002419 \\
& $s=512$ & 0.017558 & 0.002913 & 0.010818 & 0.002395 \\
\midrule
\multirow{4}{*}{OOD IC}
& Smooth, $\gamma=3.5$ & 0.011952 & 0.010559 & 0.009174 & 0.003732 \\
& Low freq., $\tau=10$ & 0.020556 & 0.002792 & 0.016768 & 0.001223 \\
& Smooth+Amp., $\gamma=3.0,\sigma=61.25$ & 0.007475 & 0.003032 & 0.004114 & 0.001092 \\
& Smooth+Low-$\tau$, $\gamma=3.5,\tau=5$ & 0.008023 & 0.006197 & 0.005505 & 0.002187 \\
\bottomrule
\end{tabularx}
\end{table*}

\subsection{Ablation Study}
\label{subsec:ablation}
\paragraph{Effect of Spatio-Temporal MeanFlow losses.}
We first conduct an ablation study on the Burgers' equation at a base resolution of $s=128$ to test the effect of spatial MeanFlow losses (S-MF) and temporal MeanFlow losses (T-MF). We define four model variants based on their loss weight configurations: a baseline without Spatio-Temporal MeanFlow regularization (w/o both), models omitting either the temporal (w/o T-MF) or spatial component (w/o S-MF), and our full framework (Ours). As presented in Tab.~\ref{tab:burgers_generalization_ablation}, omitting either regularization term leads to an increase in predictive error, showing the effectiveness of both loss components. To further analyze the model's sensitivity, a further evaluation under various combinations of T-MF and S-MF coefficients is deferred to App.~\ref{app:results}.

\textbf{Resolution Generalization.}
Benefiting from the proposed spatial MeanFlow losses, a key advantage of our framework is its potential to generalize across varying spatial resolutions. 
We therefore directly evaluate the model, trained exclusively at $s=128$, on higher-resolution grids ($s=256$/$512$) without any retraining. As shown in Tab.~\ref{tab:burgers_generalization_ablation}, our method achieves stable zero-shot generalization across varying resolutions. Notably, removing the spatial MeanFlow loss (w/o S-MF) leads to severe performance degradation, confirming its critical role in maintaining physical consistency across discretizations.

\textbf{OOD Initial Condition Generalization.}
We further test generalization to out-of-distribution initial conditions by changing the Gaussian random field parameters used to generate the Burgers initial states, while keeping the resolution fixed. 
We consider smooth, low-frequency, smooth amplified, and smooth low-$\tau$ IC shifts. As Tab.~\ref{tab:burgers_generalization_ablation} shows, our method's superior OOD performance confirms that spatio-temporal MeanFlow losses robustly handles initial-condition shifts.

\begin{figure}[ht]
    \centering
\includegraphics[width=0.4\textwidth]{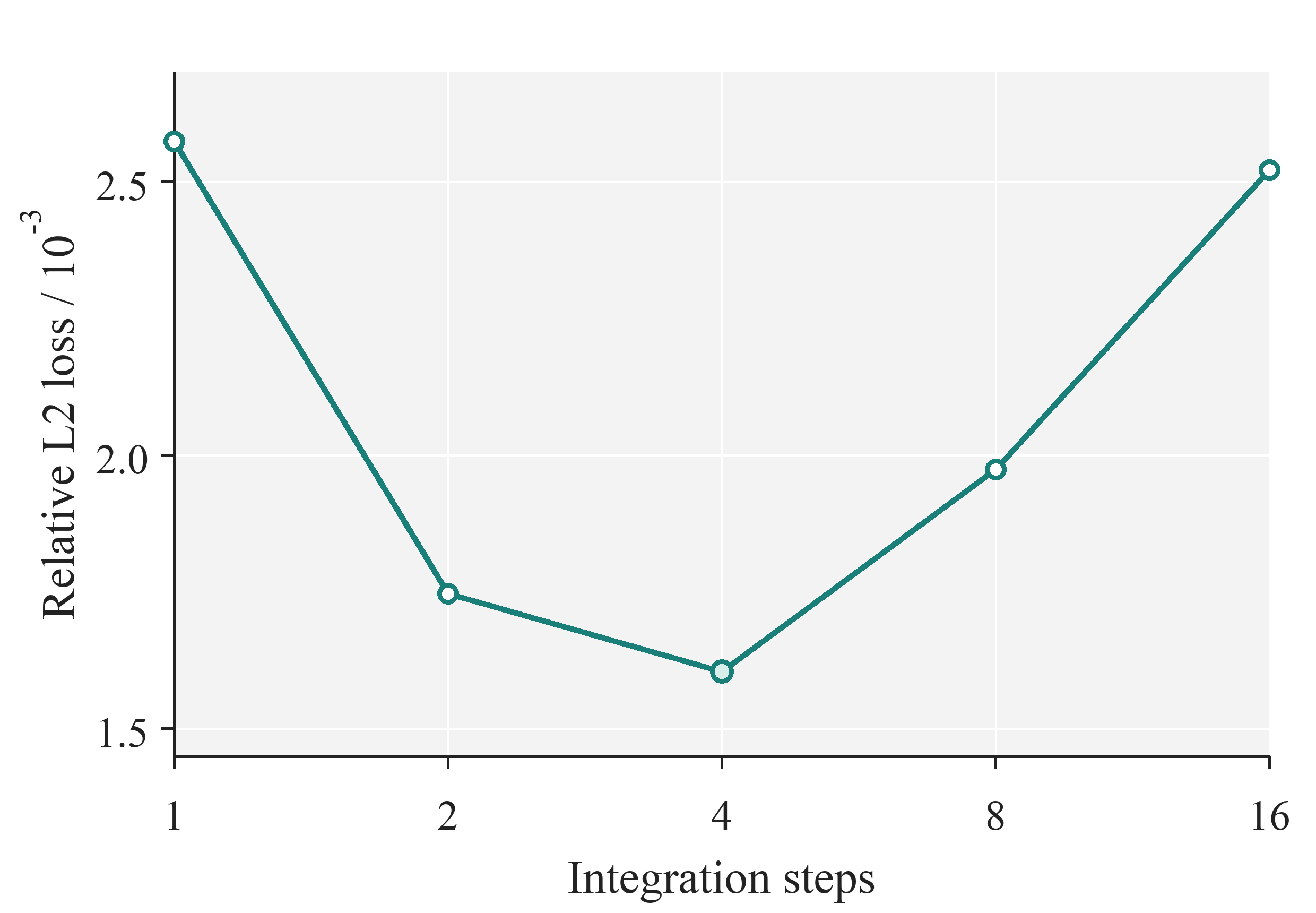}
    \caption{
    Relative $L_2$ error on Burgers with different numbers of integration steps.}
    \label{fig:burgers-step-ablation}
\end{figure}
\textbf{Integration steps.}
Since our model learns an integral operator over arbitrary time intervals,
inference does not require a fixed temporal discretization. Given an initial state
and a target time, the model can either predict the full interval in a single step
or compose several shorter integrations. We therefore evaluate the effect of the
number of integration steps on Burgers dataset. As shown in
Fig.~\ref{fig:burgers-step-ablation}, increasing the number of steps from $1$ to
$4$ reduces the relative $L_2$ error, showing that
the learned operator can be effectively reused over sub-intervals and benefits
from moderate temporal refinement. However, using too many steps degrades
performance, likely due to accumulated approximation errors from repeated
composition. This suggests that our framework provides a flexible
efficiency--accuracy trade-off: one-step inference gives fast direct prediction,
while a small number of integration steps improves accuracy without retraining.

\begin{figure*}[t]
    \centering
    \includegraphics[width=0.85\linewidth]{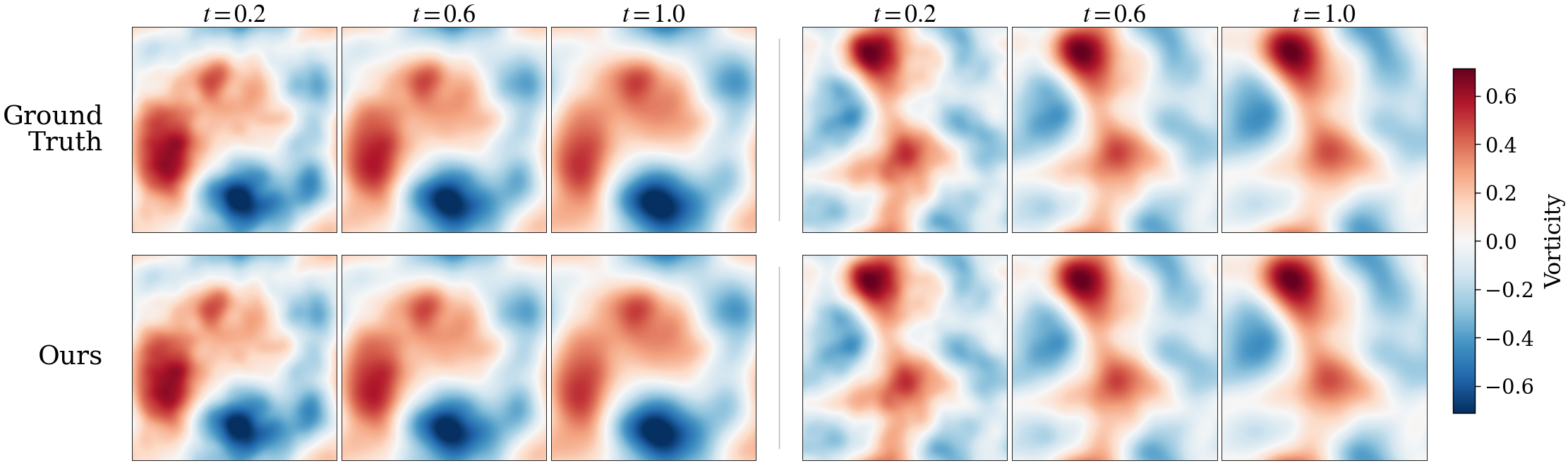}
    \caption{Qualitative visualization on the NS dataset. We compare ground-truth vorticity fields with predictions from our method at three time steps for two test samples.}
    \label{fig:ns_vis}
\end{figure*}
\subsection{Broader Applicability: Real-World Data and Inverse Modeling}
\label{subsec:broader_applicability}

To demonstrate the versatility of our Spatio-Temporal MeanFlow framework beyond standard synthetic forward prediction, we extend its evaluation to two challenging scenarios.

\textbf{Real-World Data.} We first evaluate our method on the HYCOM dataset (details in App.~\ref{app:hycom_dataset}), which provides data-assimilative 3D water temperature fields on a $32 \times 32 \times 8$ grid. A key challenge here is the absence of an explicit PDE residual due to realistic noise and assimilation artifacts. Therefore, we leverage the  Spatio-Temporal MeanFlow formulation as a weak physical regularization to learn the temporally averaged flow. This mitigates the overfitting risks of unconstrained direct regression and encourages temporally coherent predictions. As shown in Tab.~\ref{tab:main_results} (``Real-world''), this strategy ensures effective and stable predictions under realistic distributional shifts.
\textbf{Inverse Problem on Darcy Flow.} Furthermore, we evaluate whether our framework can be applied to inverse problems on Darcy flow, where the objective is to recover the unknown permeability field $\boldsymbol{a}(x)$ from observed pressure $\boldsymbol{u}^{\text{obs}}(x)$. Via test-time optimization, we freeze the pretrained forward model and optimize the log-permeability $\log \boldsymbol{a}$ as a learnable parameter using gradient descent. As shown in Tab.~\ref{tab:main_results}, our approach achieves accurate coefficient reconstruction, highlighting the broader potential of extending our framework to inverse modeling.

\section{Related Work}
 
\textbf{Physics-Informed Neural Networks (PINNs).} 
PINNs \cite{raissi2019physics} represent a foundational class of PDE solvers that parameterize solutions by embedding physical laws as soft constraints. Despite their widespread success, PINNs suffer from a stiff and pathological optimization landscape due to competing physical loss terms. While numerous strategies have been proposed to stabilize training (e.g., uncertainty-based weighting \cite{huang2021solving, kendall2018multi}), we argue that these methods rely solely on the PDE formulation itself, thereby overlooking the inherently continuous integral nature of the solving process. A more detailed discussion is provided in App. \ref{app:pinn_details}.

\textbf{Neural Operators.} 
Neural operators like FNO \cite{li2020fourier} and DeepONet \cite{lu2019deeponet} enable rapid PDE inference. To enhance physical consistency, methods like PINO \cite{li2024physics} explicitly incorporate residual constraints. For time-dependent systems, the Continuous Flow Operator (CFO) \cite{hou2026cfo} mitigates autoregressive error accumulation by repurposing flow matching to learn continuous-time dynamics. While CFO shares a similar continuous integration modeling philosophy with our approach, our framework additionally introduces a physics-inspired spatio-temporal MeanFlow loss, significantly enhancing its generalization capabilities. Other architectural variants targeting irregular geometries \cite{li2023fourier, li2020multipole} and multi-scale structures \cite{tripura2022wavelet, rahman2022u} are further discussed in App. \ref{app:operator_details}.

\textbf{Generative Modeling for PDEs.}
Generative approaches simulate PDEs by treating solution fields as samples from learned distributions. While early diffusion models \cite{huang2024diffusionpde, jacobsen2025cocogen} require numerous inference steps, recent works have improved sampling speed via rectified flows \cite{armegioiu2025rectified}, physical consistency via flow matching \cite{baldan2025flow}, and distillation \cite{zhang2025physics}. Nevertheless, these methods operate on a noise-to-solution generation paradigm. Rather than learning full data distributions, we simply adopt the continuous-time integration mechanism of generative models to formulate a continuous integrator over arbitrary time intervals. Further details are provided in App. \ref{app:generative_details}.


\section{Conclusion}

In this paper, we propose a physics-informed Neural PDE Solver via Spatio-Temporal MeanFlow. Functioning as an arbitrary-interval continuous integrator, it enables the computational efficiency of flexible integration steps. Furthermore, by enforcing physical consistency via a novel spatio-temporal loss, our method achieves enhanced generalization capabilities. Extensive evaluations demonstrate its superior accuracy, efficiency, and generalization.

\textbf{Limitations and Future Work.} 
Although our underlying mathematical formulation natively supports continuous spatial evolution, our current empirical implementation focuses on regular grids. Activating this continuous spatial modeling to tackle irregular domains represents the future step to further maximize the framework's generalization capabilities. Furthermore, extending our spatio-temporal loss to real-world datasets lacking analytical PDEs addresses a ubiquitous bottleneck in physics-informed learning. Finally, evolving our inverse modeling proof-of-concept into a unified architecture that seamlessly co-learns coefficient fields and solutions will further unleash the framework's potential.

\section*{Impact Statement}
Our Spatio-Temporal MeanFlow framework accelerates complex physical simulations by transforming traditional multi-step numerical integration into efficient, finite-interval predictions. This enhanced computational efficiency has the potential to significantly reduce the energy consumption and carbon footprint associated with large-scale scientific modeling, such as fluid dynamics, meteorology, and climate forecasting. Furthermore, its demonstrated versatility in handling real-world datasets and inverse problems opens up promising downstream applications in environmental monitoring and engineering design.

Despite these positive implications, the deployment of neural PDE solvers in safety-critical domains (e.g., aerospace design or nuclear engineering) requires careful consideration. Due to the inherent black-box nature of deep learning, neural solvers may fail unpredictably on extreme anomalies, even with improved out-of-distribution generalization. Therefore, we strongly recommend utilizing our framework to accelerate iterative design and optimization phases, while continuing to rely on mathematically guaranteed numerical methods for final safety certifications.





\bibliography{example_paper}
\bibliographystyle{icml2026}

\newpage
\appendix
\onecolumn



\section*{Appendix Overview}

\newcommand{\appitem}[2]{\item[\ref{#1}] \hyperref[#1]{#2}}

The following lists the structure of the appendix, with links to the respective sections.

\begin{itemize}
    \appitem{app:meanflow_derivation}{Detailed Derivation of the MeanFlow Identity}

    \appitem{app:related_work}{Extended Related Work}
    \begin{itemize}
        \appitem{app:pinn_details}{Physics-Informed Neural Networks (PINNs): Optimization Challenges in PINN}
        \appitem{app:operator_details}{Neural Operators: Advancements in Neural Operator Architectures}
        \appitem{app:generative_details}{Generative Modeling for PDE Simulation}
        
    \end{itemize}

    \appitem{app:generalized_algorithm}{Generalized Spatio-Temporal MeanFlow Framework}
    
    \appitem{app:theory}{Theoretical Results}
    \begin{itemize}
        \appitem{app:mf-constraints}{The Construction of Spatio-Temporal MeanFlow Constraints}
        \appitem{app:error-estimation}{Error Estimate}
        \appitem{app:decoupling-error}{MeanFlow Error Bounded by Decoupled Losses}
        \appitem{app:hessian-condition}{Hessian Condition Number Analysis}
    \end{itemize}

    \appitem{app:examples}{Examples of Spatio-Temporal MeanFlow Losse}

    \appitem{app:experiments}{Experiments}
    \begin{itemize}
        \appitem{app:data}{Details of Synthetic PDE Benchmark}
        \appitem{app:hycom_dataset}{ Details of the Real-world HYCOM Datase}
        \appitem{app:implementation}{ Implementation Details of Sec. 4.1}
        \appitem{app:hycom_implementation}{ Implementation Details for HYCO}
        \appitem{app:inverse_details}{ Implementation Details for the Inverse Darcy Experiment}
        \appitem{app:baselines}{Compared Methods}
        \appitem{app:results}{Additional Experimental Results of Sec. 4.1}
    \end{itemize}
\end{itemize}

\newpage

\section{Detailed Derivation of the MeanFlow Identity}
\label{app:meanflow_derivation}

Conventional Flow Matching parameterizes the instantaneous velocity field $\boldsymbol{v}$, which inherently requires multi-step numerical integration during inference. MeanFlow \cite{geng2025mean} seeks to directly model the average velocity $\boldsymbol{m}$ over a finite time interval, thereby enabling one-step prediction. 

Specifically, let $\boldsymbol{u}_t$ denote the state at time $t$ along the true ODE trajectory governed by $\frac{\rd\boldsymbol{u}_t}{\rd t} = \boldsymbol{v}(\boldsymbol{u}_t, t)$. The average velocity field $\boldsymbol{m}(\boldsymbol{u}_t, {\tau}, t)$ is defined as the displacement divided by the time interval $(t-{\tau})$:
$$\boldsymbol{m}(\boldsymbol{u}_t, {\tau}, t) \triangleq \frac{1}{t-{\tau}} \int_{\tau}^t \boldsymbol{v}(\boldsymbol{u}_s, s) \rd s.$$

Directly utilizing this definition as a training objective necessitates continuous numerical integration during training. To construct an integral-free training objective, we establish a differential relationship between $\boldsymbol{u}$ and $\boldsymbol{v}$. By rewriting the definition as $(t-{\tau})\boldsymbol{m}(\boldsymbol{u}_t, {\tau}, t) = \int_{\tau}^t \boldsymbol{v}(\boldsymbol{u}_s, s) \rd s$ and taking the total derivative with respect to time $t$ on both sides, we can eliminate the integral. 

For the left-hand side, applying the product rule yields:
$$\frac{\rd}{\rd t} \left[ (t-{\tau})\boldsymbol{m}(\boldsymbol{u}_t, {\tau}, t) \right] = \boldsymbol{m}(\boldsymbol{u}_t, {\tau}, t) + (t-{\tau})\frac{\rd}{\rd t}\boldsymbol{m}(\boldsymbol{u}_t, {\tau}, t).$$
The total derivative $\frac{\rd}{\rd t}\boldsymbol{m}(\boldsymbol{u}_t, {\tau}, t)$ accounts for both the explicit dependence of $\boldsymbol{m}$ on $t$ and its implicit dependence on $t$ through the evolving state $\boldsymbol{u}_t$. Using the chain rule, this expands as $\frac{\rd}{\rd t}\boldsymbol{m} = \partial_t \boldsymbol{m} + (\partial_{\boldsymbol{u}} \boldsymbol{m}) \frac{\rd\boldsymbol{u}_t}{\rd t} = \partial_t \boldsymbol{m} + (\partial_{\boldsymbol{u}} \boldsymbol{m}) \boldsymbol{v}(\boldsymbol{u}_t, t)$, assuming the reference time ${\tau}$ is a fixed constant independent of $t$ (hence $\frac{\rd{\tau}}{\rd t} = 0$).

For the right-hand side, applying the Fundamental Theorem of Calculus with respect to the upper limit of the integral yields the integrand evaluated at time $t$:
$$\frac{\rd}{\rd t} \int_{\tau}^t \boldsymbol{v}(\boldsymbol{u}_s, s) \rd s = \boldsymbol{v}(\boldsymbol{u}_t, t).$$

Equating the derivatives of both sides and rearranging the terms leads to the MeanFlow Identity:
$$\boldsymbol{m}(\boldsymbol{u}_t, {\tau}, t) = \boldsymbol{v}(\boldsymbol{u}_t, t) - (t-{\tau})\frac{\rd}{\rd t}\boldsymbol{m}(\boldsymbol{u}_t, {\tau}, t).$$

This identity provides a necessary and sufficient condition (given the natural boundary condition $\lim_{{\tau} \to t} \boldsymbol{m}(\boldsymbol{u}_t, {\tau}, t) = \boldsymbol{v}(\boldsymbol{u}_t, t)$) to characterize the true average velocity using a purely differential, integral-free form.

\section{Extended Related Work}
\label{app:related_work}

\subsection{Physics-Informed Neural Networks (PINNs): Optimization Challenges in PINNs}
\label{app:pinn_details}

Physics-Informed Neural Networks (PINNs) \cite{raissi2019physics} represent a foundational class of deep learning-based PDE solvers. They parameterize the solution manifold by embedding the governing physical laws into the training objective as soft constraints. Despite their widespread success across diverse physical domains (ranging from fluid dynamics to power systems), PINNs suffer from severe scalability bottlenecks due to their inherent reliance on instance-specific, iterative optimization. Furthermore, a prominent challenge in training PINNs is their notoriously pathological optimization landscape, caused by severe stiffness among competing boundary, initial, and residual loss terms. 

To address this problem, subsequent works have developed sophisticated adaptive schemes. Beyond standard uncertainty weighting \cite{huang2021solving} and gradient pathology mitigation \cite{kendall2018multi}, Wang et al. \cite{wang2021understanding} proposed an adaptive learning rate annealing strategy along with an improved neural architecture (IA-PINN) to balance fitting and regularization. Concurrently, methods like IAW-PINN advanced uncertainty weighting by controlling the upper bounds of weights, offering more flexible balancing among distinct loss terms. Recognizing that singular approaches are often limited by network structure, recent efforts have integrated these strengths (such as I-PINN \cite{niu2025improved}, which unifies both IA-PINN's architecture and IAW-PINN's weighting strategies). 

\subsection{Neural Operators: Advancements in Neural Operator Architectures}
\label{app:operator_details}

Neural operators learn a direct mapping between infinite-dimensional function spaces. Architectures like Fourier Neural Operators (FNO) \cite{li2020fourier} and DeepONet \cite{lu2019deeponet} enable rapid, zero-shot inference for new initial or boundary conditions. However, purely data-driven operators often struggle with out-of-distribution generalization and may violate physical laws. To address this, Physics-Informed Neural Operators (PINO) \cite{li2024physics} explicitly incorporate PDE residuals into the operator training objective, blending fast inference with physical constraints. However, when applied to time-dependent systems, these models typically rely on autoregressive (AR) schemes, which are prone to error accumulation over long horizons and restricted by uniform temporal discretization.

To address these limitations, the Continuous Flow Operator (CFO) \cite{hou2026cfo} was introduced as a novel class of flow-based operators. CFO shares our core conceptual motivation: rather than treating flow matching as a simple generative sampler, it repurposes the framework to learn the underlying differential dynamics (the physical right-hand side of the PDE) directly. By fitting temporal splines to trajectory data, CFO constructs continuous velocity fields that grant the model inherent time-resolution invariance. However, despite this shared perspective, its formulation is purely temporal; it captures the continuous evolution over time but lacks explicit constraints to ensure physical consistency across the spatial domain, which is critical for PDE generalization. Our proposed temporal-spatial MeanFlow framework strategically overcomes this. By expanding the integral constraints to the spatio-temporal domain, we explicitly enforce physical consistency across spatial resolutions. 

Furthermore, a wide array of other architectural extensions have been proposed to overcome the geometric constraints of standard neural operators. For instance, while original FNOs rely heavily on uniform grids, methods such as Geo-FNO \cite{li2023fourier} and multipole graph neural operators \cite{li2020multipole} have been developed to accurately map and compute solutions across highly irregular physical geometries. Similarly, the Galerkin Transformer \cite{cao2021choose} and CViT \cite{wang2024cvit} introduce robust mechanisms for handling non-uniform or irregularly sampled observational data. 

To capture complex, overlapping physical phenomena, models like the Wavelet Neural Operator (WNO) \cite{tripura2022wavelet} integrate multi-scale structural processing. Finally, architectures such as U-NO \cite{rahman2022u} address the depth limitations of traditional operators, employing U-Net-like structures to enable deeper, more expressive feature extraction without sacrificing computational efficiency. While these methods successfully expand the expressivity and geometric flexibility of operator learning, our approach uniquely pivots toward using generative flow matching as a direct, physical equation solver to achieve rapid, single-step integration.

\subsection{Generative Modeling for PDE Simulation}
\label{app:generative_details}

Generative modeling offers a novel paradigm for PDE simulation by treating solution fields as samples from learned distributions. Early advancements in this domain heavily relied on diffusion models. For example, DiffusionPDE \cite{huang2024diffusionpde} applies EDM-style diffusion combined with physics guidance during the sampling phase, while CoCoGen \cite{jacobsen2025cocogen} extends score-based models to ensure physically consistent generation. Although these conditional diffusion models achieve high fidelity, they typically require hundreds of stochastic steps at inference, severely limiting their practical utility for fast simulation.

To address this severe latency, researchers have increasingly explored flow matching and rectified flows. Armegioiu et al. \cite{armegioiu2025rectified} introduced a rectified flow surrogate that learns a time-dependent conditional velocity field, effectively transporting input-to-output laws along nearly straight trajectories. By leveraging a curvature-aware sampler, this approach drastically accelerates inference, reducing the required ODE evaluations to as few as eight steps while matching the Wasserstein statistics of standard diffusion models. In parallel, to improve physical fidelity within these generative frameworks, methods like PBFM \cite{baldan2025flow} have been developed to explicitly integrate physical constraints directly into flow matching training through conflict-free gradient updates. Furthermore, \cite{zhang2025physics} proposed a post-hoc distillation method to enable more accurate and fast generation for PDE problems, achieving remarkable acceleration and improvements over many PDE-related problems.

Despite these significant advancements in inference speed and physical consistency, the core objective of these models remains unchanged: they are designed to map initial noise distributions to target solution manifolds. They fundamentally operate on a noise-to-sample generation paradigm. Our approach diverges completely from this statistical sampling formulation. Rather than generating solutions from noise, we adapt the continuous-time modeling capabilities of flow matching purely from the perspective of direct equation solving. By replacing the unconstrained generative velocity field with the physical right-hand side of the PDE, our framework functions strictly as an efficient, deterministic neural integrator.

\section{Generalized Spatio-Temporal MeanFlow Framework}
\label{app:generalized_algorithm}

In the main text, we presented a simplified algorithm tailored for regular spatial grids, which fixes the spatial coordinates and focuses on continuous-time integration. Here, we present the fully generalized Spatio-Temporal MeanFlow algorithms (Alg.~\ref{alg:generalized_training} and Alg.~\ref{alg:generalized_inference}). In this generalized formulation, both the temporal and spatial coordinates are treated as continuous variables, allowing the integration path to move simultaneously across time from $\tau$ to $t$ and across space from an anchor point $\boldsymbol{\xi}$ to a target point $\boldsymbol{x}$.

\begin{figure}[htbp]
    \centering
    
    \begin{minipage}[t]{0.48\textwidth}
        \small
        \vspace{0pt} 
        \begin{algorithm}[H] 
        \small
        \caption{Generalized Training}
        \label{alg:generalized_training}
        \begin{algorithmic}[1]
        \Require PDE params $\boldsymbol{a}$, temporal indicator ${\gamma}$, network $\boldsymbol{m}_\theta$, max iter $M$, learning rate $\eta$
        \Ensure Trained network parameters $\theta^*$
        \For{$i = 1$ to $M$}
            \If{${\gamma} \neq 0$} \Comment{Time-variant PDEs}
                \State Sample anchor time $\tau$ and anchor grid $\boldsymbol{\xi}$;
                \State Compute anchor state $\boldsymbol{u}_{\textrm{anchor}}=\boldsymbol{u}(\boldsymbol{\xi}, {\tau})$;
                \State Sample target time $t>{\tau}$ and target grid $\boldsymbol{x}$;
                \State Predict MeanFlow: 
                \Statex \quad\quad $\boldsymbol{m} \gets \boldsymbol{m}_\theta([\boldsymbol{u}_{\textrm{anchor}}, \boldsymbol{\xi}, \tau, \boldsymbol{x}, t, \boldsymbol{a}])$;
                \State Compute total loss:
                \Statex \quad\quad $\mathcal{L}_{\textrm{Total}} \gets \mathcal{L}_{\textrm{Data}} \!+\! \lambda_{t}\mathcal{L}_{\textrm{T-MF}} \!+\! \lambda_{s}\mathcal{L}_{\textrm{S-MF}}$;
            \Else \Comment{Steady-state PDEs (${\gamma} = 0$)}
                \State Create initial state $\boldsymbol{u}_0$ by the harmonic
                \Statex \quad\quad extension of $\boldsymbol{u}_0(\boldsymbol{\xi})$ for $\boldsymbol{\xi}\in\partial D$; 
                \State Sample target grid $\boldsymbol{x}$ for domain $D$;
                \State Compute $\mathcal{L}_{\textrm{Data}}$ and $\mathcal{L}_{\textrm{S-MF}}$; 
                \State $\mathcal{L}_{\textrm{Total}} \gets \mathcal{L}_{\textrm{Data}} + \lambda_{s}\mathcal{L}_{\textrm{S-MF}}$; 
            \EndIf
            \State Update: $\theta \gets \theta - \eta \nabla_\theta \mathcal{L}_{\textrm{Total}}$;
        \EndFor
        \State \Return $\theta^*={\theta}$;
        \end{algorithmic}
        \end{algorithm}
    \end{minipage}%
    \hfill
    \begin{minipage}[t]{0.48\textwidth}
        \vspace{0pt} 
        \begin{algorithm}[H] 
        \small
        \caption{Generalized Multi-step Inference}
        \label{alg:generalized_inference}
        \begin{algorithmic}[1]
        \Require Pretrained net $\boldsymbol{m}_{\theta^{*}}$, start grid $\boldsymbol{x}_0$, start time $t_0$, start state $\boldsymbol{u}_0$, target grid $\boldsymbol{x}_N$, step $\Delta t$, steps $N$
        \Ensure Spatio-temporal trajectory $\{\boldsymbol{u}_k\}_{k=0}^N$
        \State Initialize trajectory list with $\boldsymbol{u}_0$;
        \For{$k = 0$ to $N-1$}
            \State Anchor coords: $(\boldsymbol{\xi}, \tau) \gets (\boldsymbol{x}_k, t_k)$;
            \State Anchor state: $\boldsymbol{u}_{\textrm{anchor}} \gets \boldsymbol{u}_k$;
            \State Target grid: $\boldsymbol{x}_{k+1} \gets \boldsymbol{x}_k + (\boldsymbol{x}_N - \boldsymbol{x}_0) / N$;
            \State Target time: $t_{k+1} \gets t_k + \Delta t$;
            \State Spatio-temporal arc length: 
            \Statex \quad\quad $\displaystyle l \gets \sqrt{\frac{\Vert \boldsymbol{x}_N-\boldsymbol{x}_0\Vert_2^2}{N^2} + {\Delta}t^2}$; 
            \State Predict MeanFlow: 
            \Statex \quad\quad $\boldsymbol{m} \gets \boldsymbol{m}_\theta([\boldsymbol{u}_{\textrm{anchor}}, \boldsymbol{\xi}, \tau, \boldsymbol{x}_{k+1}, t_{k+1}, \boldsymbol{a}])$;
            \State State increment: $\Delta\boldsymbol{u} \gets l \cdot \boldsymbol{m}$;
            \State Update field: $\boldsymbol{u}_{k+1} \gets \boldsymbol{u}_k + \Delta\boldsymbol{u}$;
            \State Append $\boldsymbol{u}_{k+1}$ to trajectory list;
        \EndFor
        \State \Return $\{\boldsymbol{u}_0, \boldsymbol{u}_1, \dots, \boldsymbol{u}_N\}$;
        \end{algorithmic}
        \end{algorithm}
    \end{minipage}
    
\end{figure}

\section{Theoretical Results}
\label{app:theory}

\subsection{The Construction of Spatio-Temporal MeanFlow Constraints}
\label{app:mf-constraints}

In this section, Prop. \ref{app:prop:mf-constraints} conclude and show how the MeanFlow constraints being constructed. 

\begin{proposition}[MeanFlow Constraints.]
\label{app:prop:mf-constraints}
    For MeanFlow function $\boldsymbol m$ satisfying
    \begin{equation}
        l(\boldsymbol {\xi}, {\tau}, \boldsymbol x, t)\cdot\boldsymbol m[\boldsymbol u(\boldsymbol {\xi}, {\tau}), \boldsymbol {\xi}, {\tau}, \boldsymbol x, t, \boldsymbol a] = \boldsymbol u(\boldsymbol x, t) - \boldsymbol u(\boldsymbol {\xi}, t), 
    \label{app:eq:fund-id}
    \end{equation}
    we have
    \begin{displaymath}
    \left\{\begin{aligned}
        &l_t\cdot\boldsymbol m = l\boldsymbol K\boldsymbol f + {\gamma}l^2\cdot \frac{\partial\boldsymbol m}{\partial {\tau}}, \\
        &l\cdot\boldsymbol K\cdot{\nabla}^\top\boldsymbol u(\boldsymbol {\xi}, {\tau}) = \boldsymbol m(\boldsymbol x - \boldsymbol {\xi})^\top - l^2\cdot\frac{\partial\boldsymbol m}{\partial \boldsymbol {\xi}}, \\
        &l\cdot\boldsymbol K\cdot{\Delta}\boldsymbol u(\boldsymbol {\xi}, {\tau}) = -(\boldsymbol {\alpha}+2\boldsymbol K^{-1}\boldsymbol {\beta}), 
    \end{aligned}\right.
    \end{displaymath}
    where $\boldsymbol {\alpha}=n\boldsymbol m-\lfrac{l_s^2}{l^2}\cdot\boldsymbol m + 2l^2{\Delta}\boldsymbol m(\boldsymbol {\xi}, {\tau})$ and $\boldsymbol {\beta}=-\frac{\partial \boldsymbol m}{\partial \boldsymbol {\xi}}(\boldsymbol x-\boldsymbol {\xi})-\lfrac{l_s^2}{l}\cdot\frac{\partial\boldsymbol m}{\partial \boldsymbol u}\boldsymbol m$. 
\end{proposition}

\begin{proof}
We enforce a perturbation on ${\gamma}$ if ${\gamma}=0$ and let ${\gamma}\to 0$ in the end phase. This technique allows us to regard parameter ${\gamma}$ as a non-sigular value. 

We compute the derivatives of $\boldsymbol m$ given that $\boldsymbol \xi , \boldsymbol x, t, \boldsymbol a$ are constants w.r.t. $\tau $ and $\tau , \boldsymbol x, t, \boldsymbol a$ are constants w.r.t. $\boldsymbol \xi $, leading to
\begin{equation}
\left\{\begin{aligned}
\frac{\rd \boldsymbol m}{\rd \tau } =& \frac{\partial \boldsymbol m}{\partial \boldsymbol u} \cdot\frac{\partial \boldsymbol u}{\partial \tau } + \frac{\partial \boldsymbol m}{\partial \tau} = \frac{\partial \boldsymbol m}{\partial \boldsymbol u}\cdot \frac{\boldsymbol f(\boldsymbol \xi , \tau )}{\gamma} + \frac{\partial \boldsymbol m}{\partial \tau}, \\
\frac{\rd \boldsymbol m}{\rd \boldsymbol \xi } =&\frac{\partial \boldsymbol m}{\partial \boldsymbol u}\cdot\frac{\partial \boldsymbol u}{\partial \boldsymbol \xi } + \frac{\partial \boldsymbol m}{\partial \boldsymbol \xi } = \frac{\partial \boldsymbol m}{\partial \boldsymbol u}\cdot\nabla^\top \boldsymbol u(\boldsymbol \xi , \tau ) + \frac{\partial \boldsymbol m}{\partial \boldsymbol \xi }. 
\end{aligned}\right.
\end{equation}

The derivatives w.r.t. $\tau $ of both sides of Eq. \eqref{app:eq:fund-id} are
\begin{equation}
\left\{\begin{aligned}
\frac{\rd \text{LHS}}{\rd \tau } =& \frac{\rd l}{\rd \tau } \boldsymbol m + l \frac{\rd \boldsymbol m}{\rd \tau } = \frac{-(t-\tau ) \boldsymbol m}{l} + l \left(\frac{\partial \boldsymbol m}{\partial \boldsymbol u}\cdot \frac{\boldsymbol f(\boldsymbol \xi , \tau )}{\gamma}+\frac{\partial \boldsymbol m}{\partial \tau}\right), \\
\frac{\rd \text{RHS}}{\rd \tau } =& -\frac{\partial \boldsymbol u(\boldsymbol \xi , \tau )}{\partial \tau } = -\frac{\boldsymbol f(\boldsymbol \xi , \tau )}{\gamma}, 
\end{aligned}\right.
\end{equation}
where $\boldsymbol f(\boldsymbol \xi , \tau )\triangleq\boldsymbol f[\boldsymbol u(\boldsymbol \xi , \tau ), \nabla , \boldsymbol \xi , \tau , \boldsymbol a]$.
Notably, the notations $\frac{\rd ~\boldsymbol \cdot}{\rd ~\boldsymbol \cdot}$ and $\frac{\partial ~\boldsymbol \cdot}{\partial ~\boldsymbol \cdot}$ both represent the Jacobian matrices, whereas $\nabla\boldsymbol f$ represent spatial gradients, which is the transpose of Jacobian matrix. 

Consequently, we have the ground-truth approximation
\begin{equation}
\begin{aligned}
\boldsymbol m_{\text{gt}} =& \frac{l}{t-\tau }\left[l \left( \frac{\partial \boldsymbol m}{\partial \boldsymbol u} \cdot\frac{\boldsymbol f(\boldsymbol \xi , \tau )}{\gamma} + \frac{\partial \boldsymbol m}{\partial \tau } \right)+\frac{\boldsymbol f(\boldsymbol \xi , \tau )}{\gamma}\right], \\
=& \frac{l}{t-\tau }\left[\boldsymbol K \cdot \frac{\boldsymbol f(\boldsymbol \xi , \tau )}{\gamma} + l \frac{\partial \boldsymbol m}{\partial \tau }\right], 
\label{eq:gt}
\end{aligned}
\end{equation}
where $l\triangleq\sqrt{(t-\tau )^2+\Vert \boldsymbol x-\boldsymbol \xi \Vert _2^2}$ and $\boldsymbol K=\boldsymbol{I}+l \frac{\partial \boldsymbol m}{\partial \boldsymbol u}$. \textbf{This coheres to the final result in conventional \textit{temporal MeanFlow}. The \textit{spatial} part is derived as follows.}

The derivatives of Eq. \eqref{app:eq:fund-id} w.r.t. $\boldsymbol \xi $, on the other hand, result in
\begin{equation}
\left\{\begin{aligned}
\frac{\rd \text{LHS}}{\rd \boldsymbol \xi } =& \frac{\boldsymbol m (\boldsymbol \xi -\boldsymbol x)^\top}{l} + l \left(\frac{\partial \boldsymbol m}{\partial \boldsymbol u} \cdot\nabla^\top \boldsymbol u(\boldsymbol \xi , \tau ) + \frac{\partial \boldsymbol m}{\partial \boldsymbol \xi }\right), \\
\frac{\rd \text{RHS}}{\rd \boldsymbol \xi } =& -\frac{\partial \boldsymbol u(\boldsymbol \xi , \tau )}{\partial \boldsymbol \xi } = -\nabla^\top \boldsymbol u(\boldsymbol \xi , \tau ),
\end{aligned}\right.
\end{equation}
and thus
\begin{equation}
    \boldsymbol K \cdot\nabla^\top \boldsymbol u(\boldsymbol \xi , \tau ) = \frac{1}{l}\boldsymbol m (\boldsymbol x -\boldsymbol {\xi})^\top - l \frac{\partial \boldsymbol m}{\partial \boldsymbol \xi }.
\label{eq:space-order1}
\end{equation}

Note that this result may substitute the $\nabla \boldsymbol u$'s of $\boldsymbol K \boldsymbol f$ in \eqref{eq:gt}. 

There are more PDEs containing two or more order of derivatives, e.g. the Navier-Stokes equations with $\Delta\boldsymbol u$ within considerations. 

For PDE with higher order in space, say order-$2$, we have
\begin{equation}
\left\{\begin{aligned}
\frac{\rd ^2 \text{LHS}}{\rd \boldsymbol \xi ^2} =& \frac{\boldsymbol m\boldsymbol 1^\top}{l} - \frac{1}{l^3}\boldsymbol m(\boldsymbol \xi -\boldsymbol x)^\top\text{diag}~(\boldsymbol \xi -\boldsymbol x) + l \left(\frac{\partial \boldsymbol m}{\partial \boldsymbol u} \cdot\frac{\partial ^2\boldsymbol u}{\partial \boldsymbol \xi ^2} + 2\frac{\partial^2 \boldsymbol m}{\partial \boldsymbol \xi ^2}\right)\\
&+ \frac{2}{l} \left(\frac{\partial \boldsymbol m}{\partial \boldsymbol u} \cdot\nabla^\top \boldsymbol u(\boldsymbol \xi , \tau ) + \frac{\partial \boldsymbol m}{\partial \boldsymbol \xi }\right)\text{diag}~(\boldsymbol \xi -\boldsymbol x), \\
\frac{\rd ^2 \text{RHS}}{\rd \boldsymbol \xi ^2} =& -\frac{\partial^2 \boldsymbol z(\boldsymbol \xi , \tau )}{\partial \boldsymbol \xi ^2},
\end{aligned}\right.
\end{equation}
where the matrices $\left[\frac{\rd ^2\boldsymbol z}{\rd \boldsymbol x^2}\right]_{ij}\triangleq\frac{\rd ^2\boldsymbol z_i}{\rd \boldsymbol x_j^2}$ and $\left[\frac{\partial ^2\boldsymbol z}{\partial \boldsymbol x^2}\right]_{ij}\triangleq\frac{\partial ^2\boldsymbol z_i}{\partial \boldsymbol x_j^2}$, then,
\begin{equation}
\begin{aligned}
&-\boldsymbol K\Delta \boldsymbol u(\boldsymbol \xi , \tau ), \\
=& -\boldsymbol K\frac{\partial ^2\boldsymbol u(\boldsymbol \xi , \tau )}{\partial ^2\boldsymbol \xi ^2}\boldsymbol 1, \\
=& -\Big(\boldsymbol I+l\frac{\partial\boldsymbol m}{\partial\boldsymbol u}\Big)\cdot\frac{\partial ^2\boldsymbol u(\boldsymbol \xi , \tau )}{\partial ^2\boldsymbol \xi ^2}\boldsymbol 1,\\
=& \frac{n}{l}\boldsymbol m - \frac{\Vert \boldsymbol \xi -\boldsymbol x\Vert_2^2}{l^3}\boldsymbol m + 2l \Delta \boldsymbol m(\boldsymbol \xi , \tau ) + \frac{2}{l} \frac{\partial \boldsymbol m}{\partial \boldsymbol \xi }(\boldsymbol \xi -\boldsymbol x) \\
&-\frac{2}{l}\frac{\partial \boldsymbol m}{\partial \boldsymbol u}\boldsymbol K^{-1}\left(\frac{1}{l}\boldsymbol m (\boldsymbol \xi -\boldsymbol x)^\top+l \frac{\partial \boldsymbol m}{\partial \boldsymbol \xi }\right)(\boldsymbol \xi -\boldsymbol x), \\
=& \frac{\boldsymbol \alpha}l - \frac{2}{l^2}\frac{\partial \boldsymbol m}{\partial \boldsymbol u}\boldsymbol K^{-1}\boldsymbol m \Vert \boldsymbol \xi -\boldsymbol x\Vert _2^2 + 2\left[\frac{\boldsymbol K}{l}-\frac{\partial \boldsymbol m}{\partial \boldsymbol u}\right] \boldsymbol K^{-1}\frac{\partial \boldsymbol m}{\partial \boldsymbol \xi }(\boldsymbol \xi -\boldsymbol x), \\
\end{aligned}
\end{equation}
hence
\begin{equation}
    \boldsymbol K\Delta \boldsymbol u(\boldsymbol \xi , \tau ) = -\Big(\frac{\boldsymbol \alpha }{l} + \frac{2}{l}\boldsymbol K^{-1}\boldsymbol \beta\Big),
\label{eq:space-order2}
\end{equation}
where $\boldsymbol \alpha \triangleq n\boldsymbol m - \frac{\Vert \boldsymbol x -\boldsymbol {\xi}\Vert _2^2}{l^2}\boldsymbol m+2l^2 \Delta \boldsymbol m(\boldsymbol \xi ,  \tau )$ and $\boldsymbol \beta \triangleq -\frac{\partial \boldsymbol m}{\partial \boldsymbol \xi }(\boldsymbol x -\boldsymbol {\xi}) - \frac{\Vert \boldsymbol x -\boldsymbol {\xi}\Vert _2^2}{l} \frac{\partial \boldsymbol m}{\partial \boldsymbol u}\boldsymbol m$, in which $n$ represents the number of spatial dimensions. 
Consequently, $-l\boldsymbol K^2 \Delta \boldsymbol u(\boldsymbol \xi , \tau ) = \boldsymbol K\boldsymbol \alpha +2\boldsymbol \beta$.
\end{proof}




\subsection{Error Estimate}
\label{app:error-estimation}

\begin{theorem}[A-posteriori Error Bound, manuscript]
\label{app:thm:error_bound}
Assume the PDE operator $\boldsymbol{f}$ is uniformly Lipschitz continuous with respect to the state variable with constant $L$, and continuously differentiable ($\mathcal{C}^1$) such that its Jacobian $\nabla_{\boldsymbol{u}}\boldsymbol{f}$ is bounded locally. Further, assume the MeanFlow predictor $\boldsymbol{m}_\theta$ and the true physical state $\boldsymbol{u}_{\textrm{True}}$ are $\mathcal{C}^1$-continuous. If the total training loss bounds the squared path residuals such that $\mathcal{L}_{\textrm{MeanFlow}} \le \epsilon$, then the pointwise error between the reconstructed state $\boldsymbol{u}_\theta$ and the true solution $\boldsymbol{u}_{\textrm{True}}$ at any target point $(\boldsymbol{x}, t)$ is bounded by:
\begin{equation}
    \|\boldsymbol{u}_\theta(\boldsymbol{x}, t) - \boldsymbol{u}_{\textrm{True}}(\boldsymbol{x}, t)\|_2 \le C \sqrt{\epsilon} + \mathcal{O}(l^2),
\end{equation}
where $C$ is a constant depending on the spatio-temporal path length $l$ and the Lipschitz constant $L$.
\end{theorem}

\begin{proof}
Let the straight spatio-temporal path $\mathfrak{p}$ connecting the starting point $(\boldsymbol{\xi}, \tau)$ and the target point $(\boldsymbol{x}, t)$ be parameterized by a scalar $s \in [0, 1]$. The predicted state along the path is linearly modeled as $\boldsymbol{u}_\theta(s) = \boldsymbol{u}(\boldsymbol{\xi}, \tau) + s \cdot l \cdot \boldsymbol{m}_\theta$, implying its path derivative is precisely $\frac{\rd \boldsymbol{u}_\theta}{\rd s} = l \cdot \boldsymbol{m}_\theta$.

We define the pointwise state error as $\boldsymbol{v}(s) = \boldsymbol{u}_\theta(s) - \boldsymbol{u}_{\textrm{True}}(s)$. Assuming the starting state is exact (i.e., anchored by data constraints such that $\boldsymbol{v}(0) = \boldsymbol{0}$), the evolution of the error along the path is governed by the difference in their derivatives:
\begin{equation}
    \frac{\rd \boldsymbol{v}}{\rd s} = \frac{\rd \boldsymbol{u}_\theta}{\rd s} - \frac{\rd \boldsymbol{u}_{\textrm{True}}}{\rd s} = l \cdot \boldsymbol{m}_\theta - l \cdot \boldsymbol{f}(\boldsymbol{u}_{\textrm{True}}(s)).
\end{equation}

To bridge the predicted MeanFlow and the true physical operator, we explicitly add and subtract the term $l \cdot \boldsymbol{f}(\boldsymbol{u}_\theta(s))$:
\begin{equation}
    \frac{\rd \boldsymbol{v}}{\rd s} = l \cdot \left[ \boldsymbol{f}(\boldsymbol{u}_\theta(s)) - \boldsymbol{f}(\boldsymbol{u}_{\textrm{True}}(s)) \right] + l \cdot \left[ \boldsymbol{m}_\theta - \boldsymbol{f}(\boldsymbol{u}_\theta(s)) \right].
\end{equation}

The first bracketed term is governed by the Lipschitz continuity of the PDE operator $\boldsymbol{f}$ with constant $L$, bounded by $L\|\boldsymbol{v}(s)\|_2$. The second bracketed term measures the deviation between the constant MeanFlow vector and the continuous local PDE evaluation. By Taylor expanding $\boldsymbol{f}(\boldsymbol{u}_\theta(s))$ around the starting point $s=0$, we relate it to the global MeanFlow residual $\boldsymbol{r}_{\textrm{MeanFlow}}$ evaluated by the network:
\begin{equation}
    l \cdot \left[ \boldsymbol{m}_\theta - \boldsymbol{f}(\boldsymbol{u}_\theta(s)) \right] = \boldsymbol{r}_{\textrm{MeanFlow}} + \boldsymbol{r}(s),
\end{equation}
where $\boldsymbol{r}(s) \sim s \cdot l^2 \nabla_{\boldsymbol{u}}\boldsymbol{f} \cdot \boldsymbol{m}_\theta$ is the first-order Taylor remainder along the path. Because the state derivatives and the Jacobian $\nabla_{\boldsymbol{u}}\boldsymbol{f}$ are assumed bounded, the remainder satisfies $\|\boldsymbol{r}(s)\|_2 \le M s l^2 = \mathcal{O}(l^2)$, with $M$ being a local bounding constant.

Taking the $L_2$ norm on both sides and applying the triangle inequality yields a rigorous differential inequality:
\begin{equation}
    \left\| \frac{\rd \boldsymbol{v}}{\rd s} \right\|_2 \le l \cdot L \|\boldsymbol{v}(s)\|_2 + \|\boldsymbol{r}_{\textrm{MeanFlow}}\|_2 + \mathcal{O}(l^2).
\end{equation}

Applying Gr\"onwall's inequality to integrate this bound from $s=0$ to $s=1$ yields:
\begin{equation}
\begin{aligned}
    &\frac{\rd }{\rd s} \left\| \boldsymbol{v}(s)\right\|_2 \le \left\| \frac{\rd \boldsymbol{v}}{\rd s} \right\|_2 \le l \cdot L \|\boldsymbol{v}(s)\|_2 + \|\boldsymbol{r}_{\textrm{MeanFlow}}\|_2 + \mathcal{O}(l^2), \\
    \Longrightarrow&\frac{\rd }{\rd s}\left[\left\| \boldsymbol{v}(s)\right\|_2\textrm{e}^{-lLs}\right]\le \left[\|\boldsymbol{r}_{\textrm{MeanFlow}}\|_2 + \mathcal{O}(l^2)\right]\textrm{e}^{-lLs}, \\
    \Longrightarrow&\|\boldsymbol{v}(1)\|_2 \le \int_0^1 \left[ \|\boldsymbol{r}_{\textrm{MeanFlow}}\|_2 + \mathcal{O}(l^2) \right] \exp\left[l L (1 - s)\right] \rd s.
\end{aligned}
\end{equation}

The integral linearly separates into two components. For the first component involving the trained residual, we apply the Cauchy-Schwarz inequality to relate it to the squared loss $\mathcal{L}_{\textrm{MeanFlow}} \le \epsilon$:
\begin{equation}
\begin{aligned}
    \int_0^1 \|\boldsymbol{r}_{\textrm{MeanFlow}}\|_2 \exp\left[l L (1 - s)\right] \rd s &\le \sqrt{\mathcal{L}_{\textrm{MeanFlow}}} \cdot \left( \int_0^1 \exp\left[2 l L (1 - s)\right] \rd s \right)^{\frac{1}{2}} \\
    &\le C \sqrt{\epsilon},
\end{aligned}
\end{equation}
where $C = \sqrt{ \frac{\exp(2lL) - 1}{2lL} }$. For the second component, the integration of the $\mathcal{O}(l^2)$ term over $s \in [0, 1]$ simply scales the exponential factor, strictly preserving the $\mathcal{O}(l^2)$ magnitude order. 

Superimposing these integrated bounds directly yields the final pointwise error bound:
\begin{equation}
    \|\boldsymbol{u}_\theta(\boldsymbol{x}, t) - \boldsymbol{u}_{\textrm{True}}(\boldsymbol{x}, t)\|_2 \le C \sqrt{\epsilon} + \mathcal{O}(l^2).
\end{equation}
This establishes that the global error is bounded by the optimization loss and a strict structural truncation error, concluding the proof.
\end{proof}

\subsection{MeanFlow Error Bounded by Decoupled Losses}
\label{app:decoupling-error}

\begin{theorem}[Thm. \ref{thm:decoupling-error}, manuscript]
\label{app:thm:decoupling-error}
Assume the MeanFlow predictor $\boldsymbol{m}_\theta$ is $\mathcal{C}^1$-continuous. If the total training loss bounds the squared path residuals such that $\mathcal{L}_{\textrm{MeanFlow}} \le \epsilon$, then the pointwise error between the reconstructed state $\boldsymbol{u}_\theta$ and the true solution $\boldsymbol{u}_{\textrm{True}}$ at any target point $(\boldsymbol{x}, t)$ is bounded by:
\begin{equation}
    \mathcal{L}_{\textrm{MeanFlow}} \le c_1 \mathcal{L}_{\textrm{Temp}} + c_2 \mathcal{L}_{\textrm{Spac}} + {\varepsilon}, 
\end{equation}
where the coefficients are defined as $c_1 = 3l_t^2$ and $c_2 = 3 \left( \frac{l - l_t}{l_s} \|\boldsymbol{g}\|_{\textrm{op}} \right)^2$. Here, ${\varepsilon}$ represents the structural coupling error and ${\varepsilon}=\mathcal O(l_s^2)$.
\end{theorem}

\begin{proof}

We begin with the spatio-temporal residual vector $\boldsymbol r_{\textrm{MeanFlow}}$:
\begin{equation}
    \boldsymbol r_{\textrm{MeanFlow}} = l_t \boldsymbol{m} - l \boldsymbol K \boldsymbol{f} - \gamma l^2 \frac{\partial \boldsymbol{m}}{\partial \tau}. \tag{S1}
\end{equation}
From the definition of the temporal residual $\boldsymbol r_{\textrm{Temp}}$, we isolate the ground-truth term:
\begin{equation}
    \boldsymbol{m} = \boldsymbol r_{\textrm{Temp}} + \boldsymbol K\boldsymbol{f} + l_t \frac{\partial \boldsymbol{m}}{\partial \tau}.
\end{equation}
Substituting this into Eq. (S1) yields:
\begin{equation}
\begin{aligned}
    \boldsymbol r_{\textrm{MeanFlow}} =& l_t (\boldsymbol r_{\textrm{Temp}} + \boldsymbol K\boldsymbol{f} + l_t \frac{\partial \boldsymbol{m}}{\partial \tau}) - l \boldsymbol K \boldsymbol{f} - \gamma l^2 \frac{\partial \boldsymbol{m}}{\partial \tau}, \\
    =& l_t \boldsymbol r_{\textrm{Temp}} + (l_t - l)\boldsymbol K\boldsymbol{f} + (l_t^2 - \gamma l^2)\frac{\partial \boldsymbol{m}}{\partial \tau}.
\end{aligned} \tag{S2}
\end{equation}
Given the linear gradient PDE structure $\boldsymbol{f} = \nabla^\top \boldsymbol{u} \cdot \boldsymbol{g} + \boldsymbol{h}$, and the spatial residual $\boldsymbol r_{\textrm{Spac}}$, we substitute the spatial gradient $\boldsymbol K\nabla^\top \boldsymbol{u}$ into the expression for $\boldsymbol K\boldsymbol{f}$, i.e., 
\begin{equation}
\begin{aligned}
    &\boldsymbol r_{\textrm{MeanFlow}} = l_t \boldsymbol r_{\textrm{Temp}} + (l_t - l)\boldsymbol K({\nabla}^\top\boldsymbol u\cdot\boldsymbol g+\boldsymbol h) + (l_t^2 - \gamma l^2)\frac{\partial \boldsymbol{m}}{\partial \tau}, \\
    =& l_t \boldsymbol r_{\textrm{Temp}} + \frac{l_t-l}{l_s}(\boldsymbol r_\textrm{Spac}+\boldsymbol m(\boldsymbol x-\boldsymbol {\xi})^\top-l_s^2\frac{\partial\boldsymbol m}{\partial \boldsymbol {\xi}})\cdot\boldsymbol g + (l_t - l)\boldsymbol K\boldsymbol h + (l_t^2 - \gamma l^2)\frac{\partial \boldsymbol{m}}{\partial \tau}, \\
\end{aligned} \tag{S3}
\end{equation}
Applying the generalized triangle inequality for squared $L_2$ norms (or Jensen's inequality for quadratic function, i.e., $\|\sum_{i=1}^3 \boldsymbol{a}_i\|_2^2 \le 3\sum_{i=1}^3 \|\boldsymbol{a}_i\|_2^2$) and the operator norm bound $\|\boldsymbol{v} \cdot \boldsymbol{g}\|_2 \le \|\boldsymbol{g}\|_{\textrm{op}} \|\boldsymbol{v}\|_2$, we obtain:
\begin{equation}
    \|\boldsymbol r_{\textrm{MeanFlow}}\|_2^2 \le 3l_t^2 \|\boldsymbol r_{\textrm{Temp}}\|_2^2 + 3 \left( \frac{l - l_t}{l_s} \|\boldsymbol{g}\|_{\textrm{op}} \right)^2 \|\boldsymbol r_{\textrm{Spac}}\|_2^2 + 3\|\boldsymbol{r}\|_2^2
\end{equation}
where
\begin{equation}
    \boldsymbol{r} = \underbrace{\frac{l_t-l}{l_s}\left[\boldsymbol m(\boldsymbol x-\boldsymbol {\xi})^\top - l_s^2\frac{\partial\boldsymbol m}{\partial \boldsymbol {\xi}}\right]\cdot\boldsymbol g}_\textrm{Boundary or Higher-Order Discretization Errors} + \underbrace{(l_t-l)\boldsymbol K\boldsymbol h}_\textrm{Source Term Error} + \underbrace{(l_t^2-{\gamma}l^2)\frac{\partial\boldsymbol m}{\partial {\tau}}}_\textrm{Temporal Scaling Mismatch}, 
\end{equation}
collects the remaining structural terms. Therefore, as $\Vert \boldsymbol K\Vert _\textrm{op}\le {\rho}$ and the state variables are sufficiently smooth such that their norms are bounded, we have, applying generalized triangle inequality for squared $L_2$ norms again, 
\begin{equation}
\begin{aligned}
    {\varepsilon} = 3\Vert \boldsymbol r\Vert ^2_2 \le&  9(l_t-l)^2\cdot\left[\frac1 {l_s^2}\cdot\left\Vert \boldsymbol r_{\boldsymbol g}\right\Vert ^2_2 + {\rho}^2\Vert \boldsymbol h\Vert _2^2 + \left(\frac{l_t^2 - \gamma l^2}{l_t-l}\right)^2 \left\Vert \frac{\partial \boldsymbol{m}}{\partial \tau} \right\Vert_2^2\right], \\
    =& \left\{\begin{array}{ll}9\left\Vert \boldsymbol r_{\boldsymbol g}\right\Vert ^2_2 + 9l_s^2\cdot{\rho}^2\Vert \boldsymbol h\Vert _2^2,&\text{if }{\gamma} = 0, \\ 9(l_t-l)^2\cdot\left[\frac1 {l_s^2}\cdot\left\Vert \boldsymbol r_{\boldsymbol g}\right\Vert ^2_2 + {\rho}^2\Vert \boldsymbol h\Vert _2^2 + (l_t+l)^2 \left\Vert \frac{\partial \boldsymbol{m}}{\partial \tau} \right\Vert_2^2\right],&\text{if }{\gamma}=1,\end{array}\right.
\end{aligned}
\end{equation}
where $\boldsymbol r_{\boldsymbol g} \triangleq \left[\boldsymbol m(\boldsymbol x-\boldsymbol {\xi})^\top-l_s^2\frac{\partial\boldsymbol m}{\partial \boldsymbol {\xi}}\right] \cdot\boldsymbol g$. Here, the result for $\gamma = 0$ follows from the fact that $l_t = 0$ and $l = l_s$ for a steady-state system implied by $\gamma = 0$. 

Note that the higher-order discretization error $\boldsymbol r_{\boldsymbol g}$ satisfies
\begin{displaymath}
\begin{aligned}
    \Vert \boldsymbol r_{\boldsymbol g}\Vert _2\le&\Vert \boldsymbol m(\boldsymbol x-\boldsymbol {\xi})^\top\boldsymbol g\Vert _2 + \left\Vert l_s^2\frac{\partial\boldsymbol m}{\partial \boldsymbol {\xi}}\boldsymbol g\right\Vert _2 \le \Vert \boldsymbol m\Vert _2\cdot\Vert \boldsymbol x-\boldsymbol {\xi}\Vert _2\cdot\Vert \boldsymbol g\Vert _\textrm{op} + \mathcal O(l_s^2), \\
    =& \Vert \boldsymbol m\Vert _2\cdot l_s\cdot\Vert \boldsymbol g\Vert _\textrm{op} = \mathcal O(l_s). 
\end{aligned}
\end{displaymath}

Consequently, since the bracketed term in the $\gamma=1$ case is bounded by $\mathcal{O}(1)$ and the temporal scaling factor satisfies $l_t-l=\frac{l_t^2-l^2}{l_t+l}=-\frac{l_s^2}{l_t+l}=\mathcal O(l_s^2)$, the asymptotic bound for the structural coupling error yields:
\begin{displaymath}
    {\varepsilon} = \left\{\begin{array}{ll}\mathcal O(l_s^2),&\text{if }{\gamma} = 0, \\ \mathcal O(l_s^4),&\text{if }{\gamma}=1.\end{array}\right.
\end{displaymath}

This concludes the proof.
\end{proof}

\subsection{Hessian Condition Number Analysis}
\label{app:hessian-condition}

\begin{theorem}[Hessian Condition Number Analysis]
\label{app:thm:hessian}
Let $\mathcal{H}_{\textrm{Coupled}}$ and $\mathcal{H}_{\textrm{Decoupled}}$ be the Hessian matrices of the coupled spatio-temporal loss and the disentangled loss with respect to the network parameters $\boldsymbol{\theta} \in \mathbb{R}^p$. Under inherent physical scale-mismatches between spatial and temporal dimensions, the condition numbers satisfy:
\begin{equation}
    \kappa(\mathcal{H}_{\textrm{Decoupled}}) \ll \kappa(\mathcal{H}_{\textrm{Coupled}}).
\end{equation}
\end{theorem}

\begin{proof}
Consider the coupled loss optimized over the neural network parameters $\boldsymbol{\theta}$. Based on the linear decomposition in Eq. (S3), Thm. \ref{app:thm:decoupling-error}, the coupled residual can be abstractly expressed as $\boldsymbol{r}_{\textrm{MeanFlow}}(\boldsymbol{\theta}) \approx \alpha \boldsymbol{r}_{\textrm{Temp}}(\boldsymbol{\theta}) + \beta \boldsymbol{r}_{\textrm{Spac}}(\boldsymbol{\theta})$, where $\alpha \sim \mathcal{O}(l_t)$ and $\beta \sim \mathcal{O}(l_s^{-1})$ denote the characteristic temporal and spatial scales. The coupled loss is given by:
\begin{equation}
\begin{aligned}
    \mathcal{L}_{\textrm{Coupled}}(\boldsymbol{\theta}) &= \frac{1}{2} \| \alpha \boldsymbol{r}_{\textrm{Temp}} + \beta \boldsymbol{r}_{\textrm{Spac}} \|_2^2 \\
    &= \underbrace{\frac{1}{2} \alpha^2 \| \boldsymbol{r}_{\textrm{Temp}} \|_2^2 + \frac{1}{2} \beta^2 \| \boldsymbol{r}_{\textrm{Spac}} \|_2^2}_{\mathcal{L}_{\textrm{Decoupled}}(\boldsymbol{\theta})} + \underbrace{\alpha \beta \langle \boldsymbol{r}_{\textrm{Temp}}, \boldsymbol{r}_{\textrm{Spac}} \rangle}_{\mathcal{L}_{\textrm{Cross}}(\boldsymbol{\theta})}.
\end{aligned}
\end{equation}

Let $\boldsymbol{J}_T = \nabla_{\boldsymbol{\theta}} \boldsymbol{r}_{\textrm{Temp}} \in \mathbb{R}^{n \times p}$ and $\boldsymbol{J}_S = \nabla_{\boldsymbol{\theta}} \boldsymbol{r}_{\textrm{Spac}} \in \mathbb{R}^{n \times p}$ be the Jacobian matrices of the temporal and spatial residuals, respectively. Applying the Gauss-Newton approximation (neglecting second-order residual terms), the Hessian of the decoupled loss is a well-conditioned sum of two positive semi-definite (PSD) matrices:
\begin{equation}
    \mathcal{H}_{\textrm{Decoupled}} \approx \alpha^2 \boldsymbol{J}_T^\top \boldsymbol{J}_T + \beta^2 \boldsymbol{J}_S^\top \boldsymbol{J}_S.
\end{equation}

Conversely, the Hessian of the coupled loss introduces explicit cross-Jacobian interactions:
\begin{equation}
    \mathcal{H}_{\textrm{Coupled}} \approx \mathcal{H}_{\textrm{Decoupled}} + \underbrace{\alpha \beta \left( \boldsymbol{J}_T^\top \boldsymbol{J}_S + \boldsymbol{J}_S^\top \boldsymbol{J}_T \right)}_{\mathcal{H}_{\textrm{Cross}}}.
\end{equation}

In stiff physical systems, temporal variations and spatial gradients evolve at vastly mismatched scales, leading to highly skewed singular values between $\boldsymbol{J}_T$ and $\boldsymbol{J}_S$. The matrix $\mathcal{H}_{\textrm{Cross}}$ is highly indefinite and dense. According to Weyl's inequality for Hermitian matrices, the minimum eigenvalue of the coupled Hessian is bounded by:
\begin{equation}
    \lambda_{\min}(\mathcal{H}_{\textrm{Coupled}}) \le \lambda_{\min}(\mathcal{H}_{\textrm{Decoupled}}) + \lambda_{\max}(\mathcal{H}_{\textrm{Cross}}).
\end{equation}
Because $\mathcal{H}_{\textrm{Cross}}$ is indefinite and dominated by the scale mismatch $\mathcal{O}(\alpha \beta)$, its extreme eigenvalues perturb the spectrum of $\mathcal{H}_{\textrm{Coupled}}$ violently. This interference dramatically suppresses $\lambda_{\min}(\mathcal{H}_{\textrm{Coupled}})$ toward zero (or creates pathological ravines in the optimization landscape) while inflating $\lambda_{\max}$, thereby severely escalating the condition number:
\begin{equation}
    \kappa(\mathcal{H}_{\textrm{Coupled}}) = \frac{\lambda_{\max}(\mathcal{H}_{\textrm{Coupled}})}{\lambda_{\min}(\mathcal{H}_{\textrm{Coupled}})} \gg \frac{\lambda_{\max}(\mathcal{H}_{\textrm{Decoupled}})}{\lambda_{\min}(\mathcal{H}_{\textrm{Decoupled}})} = \kappa(\mathcal{H}_{\textrm{Decoupled}}).
\end{equation}
By optimizing $\mathcal{L}_{\textrm{Temp}}$ and $\mathcal{L}_{\textrm{Spac}}$ independently, the decoupled strategy explicitly annihilates the pathological cross-term $\mathcal{H}_{\textrm{Cross}}$, ensuring a geometrically stable and well-conditioned gradient flow.
\end{proof}

\section{Examples of Spatio-Temporal MeanFlow Losses}
\label{app:examples}

\textbf{Example 1: d’Alembert’s equation} The proposed method applys to the equations
\begin{displaymath}
\left\{\begin{aligned}
&\frac{\partial \boldsymbol u}{\partial t} + \frac{\partial \boldsymbol u}{\partial \boldsymbol x}\boldsymbol c(\boldsymbol u) = \boldsymbol 0, \\
&\boldsymbol u(\boldsymbol x, 0) = \boldsymbol u_0(\boldsymbol x), 
\end{aligned}\right.
\end{displaymath}
the variables in our method are $\nabla^\top \boldsymbol u=\frac{\partial\boldsymbol u}{\partial\boldsymbol x}, \boldsymbol f=-\frac{\partial\boldsymbol u}{\partial\boldsymbol x}\boldsymbol c(\boldsymbol u)$, hereby Eq. \eqref{eq:gt} obtains 
\begin{displaymath}
\begin{aligned}
&(t-\tau ) \boldsymbol m - \left(-\boldsymbol m (\boldsymbol x -\boldsymbol {\xi})^\top + l^2 \frac{\partial \boldsymbol m}{\partial \boldsymbol \xi }\right) \boldsymbol c(\boldsymbol u) - l^2 \frac{\partial \boldsymbol m}{\partial \tau } = \boldsymbol 0, \\
\iff& \left[t-\tau +(\boldsymbol x -\boldsymbol {\xi})^\top\boldsymbol c(\boldsymbol u)\right] \boldsymbol m - l^2 \left(\frac{\partial \boldsymbol m}{\partial \boldsymbol \xi } \boldsymbol c(\boldsymbol u) + \frac{\partial \boldsymbol m}{\partial \tau }\right) = \boldsymbol 0,
\end{aligned}
\end{displaymath}
where $l^2 = (t-\tau )^2+\Vert \boldsymbol x -\boldsymbol {\xi}\Vert _2^2$ and the equation allows the construction of a MeanFlow loss just as the original MeanFlow framework did. 

\textbf{Example 2(a): Navier-Stokes equations} The Navier-Stokes equations are, 
\begin{equation}
\left\{\begin{aligned}
& \frac{\partial \boldsymbol w(\boldsymbol x, t)}{\partial t} = - \frac{\partial \boldsymbol w(\boldsymbol x, t)}{\partial \boldsymbol x}\cdot \boldsymbol u(\boldsymbol x, t) + v\Delta \boldsymbol w(\boldsymbol x, t) + \boldsymbol f(\boldsymbol x), \\
&\boldsymbol w(\boldsymbol x, t) = \nabla \times\boldsymbol u(\boldsymbol x, t), \\
&\nabla \cdot \boldsymbol u(\boldsymbol x, t) = 0, \\
&\boldsymbol w(\boldsymbol x, 0) = \boldsymbol w_0(\boldsymbol x), \forall x\in D.
\end{aligned}\right.
\label{eq:N-S}
\end{equation}
We can eliminate $\boldsymbol u(\boldsymbol x, t)$ by solving
\begin{displaymath}
\boldsymbol u(\boldsymbol x, t) = \frac{1}{4\pi }\int_{\mathbb R^3}\frac{\boldsymbol w(\boldsymbol y, t)\times(\boldsymbol x-\boldsymbol y)}{\Vert \boldsymbol x-\boldsymbol y\Vert _2^3}\rd \boldsymbol y,
\end{displaymath}
using the Biot–Savart integration under the condition that $\nabla \cdot \boldsymbol u=0$.

\section{Experiments}
\label{app:experiments}

\subsection{Details of Synthetic PDE Benchmarks}
\label{app:data}

We follow the data generation protocol of DiffusionPDE~\citep{huang2024diffusionpde} and use four synthetic PDE benchmarks: 1D Burgers' equation, 2D non-bounded Navier--Stokes equation, 2D Darcy flow, and 2D Poisson equation. For each benchmark, we use 10,000 samples and split them into training, validation, and test sets with a ratio of 7:2:1, resulting in 7,000 training samples, 2,000 validation samples, and 1,000 test samples.

\subsubsection{1D Burgers' Equation}

\paragraph{Governing equation.}
We consider the one-dimensional viscous Burgers' equation on \(\Omega=(0,1)\) with periodic boundary conditions:
\[
\partial_t u(c,t) + \partial_c\left(\frac{u^2(c,t)}{2}\right)
=
\nu \partial_{cc}u(c,t),
\quad c\in\Omega,\quad t\in(0,T].
\]
The initial condition is
\[
u(c,0)=u_0(c),
\]
and the viscosity is set to \(\nu=0.01\).

\paragraph{Numerical discretization and data generation.}
Following DiffusionPDE, the spatial resolution is \(128\), and each trajectory contains \(128\) temporal states, including the initial state and 127 future time steps. The time interval is \(T=1\). Initial conditions are sampled from a one-dimensional Gaussian random field with parameters \(\gamma=2.5\), \(\tau=7\), and \(\sigma=7^2\). The equation is solved using the spectral time-stepping solver used in the DiffusionPDE/FNO-style Burgers data generation code. The resulting data are stored as trajectories of shape \(N\times X\times T\), where \(N=10{,}000\), \(X=128\), and \(T=128\).

\paragraph{Use in our method.}
For Burgers' equation, our model takes the initial condition \(u_0(c)\) as input and predicts the full spatio-temporal solution trajectory \(u(c,t)\). 

\subsubsection{2D Non-bounded Navier--Stokes Equation}

\paragraph{Governing equation.}
We consider the two-dimensional incompressible Navier--Stokes equation in vorticity form:
\[
\partial_t \omega(c,t)
+
\boldsymbol{v}(c,t)\cdot \nabla \omega(c,t)
=
\nu \Delta \omega(c,t)
+
q(c),
\quad c\in\Omega,\quad t\in(0,T],
\]
with the incompressibility constraint
\[
\nabla\cdot \boldsymbol{v}(c,t)=0.
\]
Here, \(\omega=\nabla\times \boldsymbol{v}\) is the vorticity, \(\boldsymbol{v}\) is the velocity field, and \(q(c)\) is the forcing term. We set \(\nu=10^{-3}\), corresponding to Reynolds number \(\mathrm{Re}=1000\).

\paragraph{Numerical discretization and data generation.}
The spatial domain is discretized on a \(128\times128\) grid. Following DiffusionPDE, the initial vorticity field is sampled from a two-dimensional Gaussian random field with \(\alpha=2.5\) and \(\tau=7\). The forcing term is fixed as
\[
q(c)=0.1\left(\sin(2\pi(x+y))+\cos(2\pi(x+y))\right).
\]
The equation is solved in Fourier space with a pseudo-spectral method, de-aliasing, and a Crank--Nicolson update. The internal time step is \(10^{-4}\), the final time is \(T=1\), and 10 snapshots are recorded for each trajectory. We use ten files, each containing 1,000 samples, giving 10,000 samples in total. Each file stores the initial vorticity field with shape \(1000\times128\times128\) and the solution trajectory with shape \(1000\times128\times128\times10\).

\paragraph{Use in our method.}
For Navier--Stokes, our model takes the initial vorticity field \(\omega(c,0)\) as input and predicts the future vorticity fields over the recorded time steps. The forcing term and boundary/domain setting are fixed across the dataset.

\subsubsection{2D Darcy Flow}
\label{subsec:darcy_flow}

\paragraph{Governing equation.}
We consider the steady-state Darcy flow equation on \(\Omega=(0,1)^2\):
\[
-\nabla\cdot\left(a(\mathbf{c})\nabla u(\mathbf{c})\right)=q(\mathbf{c}),
\quad \mathbf{c}\in\Omega,
\]
with homogeneous Dirichlet boundary condition
\[
u(\mathbf{c})=0,\quad \mathbf{c}\in\partial\Omega.
\]
Here, \(a(\mathbf{c})\) is the permeability coefficient, \(u(\mathbf{c})\) is the pressure field, and the forcing term is fixed as \(q(\mathbf{c})=1\).

\paragraph{Numerical discretization and data generation.}
The domain is discretized on a \(128\times128\) grid. Following DiffusionPDE, a Gaussian random field \(g(\mathbf{c})\) is sampled with parameters \(\alpha=2\) and \(\tau=3\), and the log-normal permeability field is constructed as
\[
a(\mathbf{c})=\exp(g(\mathbf{c})).
\]
The PDE is solved by assembling a sparse finite-difference linear system for the elliptic operator with zero boundary values. We use 10,000 Darcy samples. The coefficient field has shape \(10000\times128\times128\), and the corresponding pressure solution has the same spatial resolution.

\paragraph{Use in our method.}
For Darcy flow, our model is conditioned on the permeability field \(a(\mathbf{c})\) and the prescribed boundary condition, and predicts the pressure solution \(u(\mathbf{c})\). Since the boundary condition is homogeneous and shared by all samples, it is treated as fixed conditioning in the benchmark.

\subsubsection{2D Poisson Equation}

\paragraph{Governing equation.}
We consider the two-dimensional Poisson equation on \(\Omega=(0,1)^2\):
\[
\Delta u(c)=q(c),
\quad c\in\Omega,
\]
with homogeneous Dirichlet boundary condition
\[
u(c)=0,\quad c\in\partial\Omega.
\]
Here, \(q(c)\) is the source field and \(u(c)\) is the solution field.

\paragraph{Numerical discretization and data generation.}
The domain is discretized on a \(128\times128\) grid. Following DiffusionPDE, the source field \(q(c)\) is sampled from a Gaussian random field with parameters \(\alpha=2\) and \(\tau=3\). The Poisson equation is solved using a finite-difference discretization with a standard five-point Laplacian stencil and homogeneous Dirichlet boundary values. We use 10,000 samples. The source field has shape \(10000\times128\times128\), and the corresponding solution field has the same spatial resolution.

\paragraph{Use in our method.}
For the Poisson equation, our model is conditioned on the source field \(q(c)\) and the prescribed boundary condition, and predicts the solution field \(u(c)\). As in Darcy flow, the boundary condition is fixed for all samples.

\subsection{Details of the Real-world HYCOM Dataset}
\label{app:hycom_dataset}

We use a real-world ocean dynamics dataset derived from HYCOM for evaluating the scalability and robustness of the proposed method. HYCOM, the Hybrid Coordinate Ocean Model, is a global ocean analysis and forecasting system. Different from purely synthetic PDE benchmarks, HYCOM products are generated by combining numerical ocean model dynamics with real observations through data assimilation. The assimilated observations include satellite measurements and in-situ ocean observations. Therefore, the resulting fields reflect realistic ocean variability, while also containing practical complications such as observation noise, assimilation artifacts, imperfect boundaries, and possible mismatch between the observed dynamics and idealized PDE assumptions.

\paragraph{Data source.}
The data are obtained from the public HYCOM GLBy0.08 product through the THREDDS/OPeNDAP service:
\[
\texttt{https://tds.hycom.org/thredds/dodsC/GLBy0.08/latest}.
\]
We use the variable \texttt{water\_temp}, which represents three-dimensional ocean water temperature. The original HYCOM product provides global ocean fields over longitude, latitude, depth, and time. In our experiment, we select a local ocean region and use the first 8 depth levels to construct a three-dimensional spatio-temporal prediction task.

\paragraph{Selected region and variables.}
We extract a local spatial block from the HYCOM grid using latitude indices $[1100,1260)$ and longitude indices $[2900,3060)$. This corresponds approximately to latitudes from $-36.0^\circ$ to $-29.64^\circ$ and longitudes from $232.0^\circ$ to $244.72^\circ$ in the HYCOM convention, i.e., approximately $128.0^\circ$W to $115.28^\circ$W. The selected depth levels are
$
0,\ 2,\ 4,\ 6,\ 8,\ 10,\ 12,\ 15 $ meters.

The raw extracted block has shape
$
32 \times 8 \times 160 \times 160,
$
corresponding to time, depth, latitude, and longitude dimensions, respectively.
\paragraph{Preprocessing.}
We normalize the extracted temperature field using the mean and standard deviation computed over the selected raw block. The statistics used in our processed dataset are
$
\mu = 23.01293796527723, \qquad
\sigma = 1.3310423861925285.
$
The raw temperature values in the selected block range approximately from
$
18.8700 \text{ to } 26.1490.
$
After normalization, we construct local spatio-temporal samples using a sliding-window procedure. Each sample contains 9 time steps on a spatial grid of size
$
32 \times 32 \times 8.
$
The spatial stride is set to 4 along the two horizontal dimensions. This produces a total of
$
26136
$
samples, each with shape
$
9 \times 32 \times 32 \times 8.
$

\paragraph{Dataset split.}
We split the processed samples into training, validation, and test sets using an 80/10/10 split:
$
20908 \text{ training samples}, \quad
2614 \text{ validation samples}, \quad
2614 \text{ test samples}.
$
All methods are trained and evaluated on the same split.

\paragraph{Prediction task.}
Given the initial three-dimensional temperature field and the spatio-temporal coordinates, the model predicts the future temperature evolution over the 9-step temporal window. Unlike idealized synthetic PDE data, the HYCOM data are not generated from a single clean closed-form PDE with perfectly specified boundary conditions. Instead, they are data-assimilative ocean fields influenced by numerical model dynamics, observations, and assimilation procedures. This makes the benchmark more representative of real deployment settings, where data may be noisy, partially observed, and only approximately consistent with simplified PDE assumptions.

\subsection{Implementation Details of Sec. \ref{subsec:forward}}
\label{app:implementation}

\begin{table}[t]
\centering
\small
\setlength{\tabcolsep}{4.5pt}
\renewcommand{\arraystretch}{1.12}
\caption{Hyperparameter settings for synthetic PDE benchmarks.}
\label{tab:hyperparameters}
\begin{tabular}{lccccc}
\toprule
Dataset & Backbone & Batch size & Learning rate & Epochs & Loss weights \\
\midrule
Burgers & Conditioned FNO1d & 8 & \(5.5\times10^{-4}\) & 200 & \(\lambda_t=\lambda_x=10^{-2}\)  \\
Navier--Stokes & Fourier U-Net & 5 & \(2.0\times10^{-4}\) & 100 & \(\lambda_t=\lambda_x=10^{-2}\) \\
Darcy & Fourier U-Net & 4 & \(3.0\times10^{-4}\) & 300 & \(\lambda_x=10^{-2}\) \\
Poisson & Fourier U-Net & 4 & \(3.0\times10^{-4}\) & 300 & \(\lambda_x=10^{-2}\) \\
\bottomrule
\end{tabular}
\end{table}

For all PDEs, four finite-interval pairs \(({\tau},t)\) are randomly sampled at each iteration to construct the MeanFlow regularization. For steady-state PDEs, the model takes the coefficient/source field together with the prescribed boundary condition as input. Models are trained with a supervised relative \(L_2\) loss and the proposed MeanFlow regularization. All results are reported using the checkpoint with the lowest validation relative \(L_2\) error.

For Burgers' equation, we use a conditioned FNO1d. The model takes the current solution state and spatial coordinate as input, and is conditioned on the interval information. It uses Fourier modes \([12,12,12,12,12]\), channel width \(96\), fully connected hidden dimension \(192\), GELU activation, and a conditioning MLP with hidden dimension \(224\).

For Navier--Stokes, Darcy flow, and Poisson equation, we use a Fourier U-Net backbone. The Fourier U-Net follows an encoder--decoder architecture with skip connections. The initial feature projection is a \(3\times3\) convolution with \(48\) hidden channels. The encoder has four resolution levels with channel multipliers \((1,2,2,4)\), and each level contains two residual blocks. In the first two encoder levels, the standard convolutional residual blocks are replaced by Fourier residual blocks, where each block combines spectral convolution layers with local \(1\times1\) convolution branches. The remaining encoder blocks, the middle block, and the decoder use convolutional residual blocks. The decoder upsamples features back to the original resolution and concatenates skip features from the encoder. We use GELU activation, no attention blocks, and no normalization.

For Navier--Stokes, the Fourier U-Net input contains the current vorticity field and the two spatial coordinate channels. For Darcy flow and Poisson equation, the Fourier U-Net uses two scalar input channels: the harmonic extension of the boundary condition and the coefficient/source field. Darcy flow and Poisson equation use Fourier modes \((8,8)\) in the first Fourier level, while Navier--Stokes uses the default Fourier modes \((12,12)\). The number of retained modes is reduced at lower resolutions according to the U-Net scale. Table~\ref{tab:hyperparameters} summarizes the main hyperparameters used in our experiments.

\subsection{Implementation Details for HYCOM}
\label{app:hycom_implementation}

This section provides implementation details for applying our MeanFlow-based solver to the real-world HYCOM ocean dataset. The goal of this experiment is to evaluate whether the proposed framework can scale to realistic three-dimensional spatio-temporal data and remain effective when the observed dynamics are not generated from an ideal synthetic PDE.

\paragraph{Input and output format.}
Each processed HYCOM sample has shape
$
9 \times 32 \times 32 \times 8,
$
where the first dimension denotes time and the remaining dimensions denote the three-dimensional spatial grid. We denote the normalized temperature field as
$
u \in \mathbb{R}^{T \times X \times Y \times Z},
$
with $T=9$, $X=32$, $Y=32$, and $Z=8$. In implementation, the data are rearranged to
$
u \in \mathbb{R}^{X \times Y \times Z \times T}.
$
The initial state $u_0$ is used as the conditioning field, and the model predicts the future temperature evolution over the full temporal window.

\paragraph{Model input.}
For each query time step, we construct a five-channel input tensor
$
\mathbf{x} = (x, y, z, t, u_0),
$
where $(x,y,z)$ are normalized spatial coordinates, $t$ is the normalized temporal coordinate, and $u_0$ is the initial three-dimensional temperature field broadcast to all query times. Thus, for each time step, the model input has shape
$
32 \times 32 \times 8 \times 5.
$
The model outputs a scalar temperature-related quantity at each spatial location.

\paragraph{Network architecture.}
We use a three-dimensional Fourier neural operator backbone. The model applies spectral convolution over the three spatial dimensions and predicts the temperature evolution conditioned on the initial state and coordinates. Unless otherwise specified, the HYCOM experiment uses the following architecture:
$
\text{modes}_x = [8,8,8,8], \quad
\text{modes}_y = [8,8,8,8], \quad
\text{modes}_z = [4,4,4,4],
$
with channel width $20$, fully connected dimension $64$, and hidden layers
$
[20,20,20,20,20].
$
The activation function is GELU. No additional padding is used.

\paragraph{Training protocol.}
We train the model for 101 epochs using Adam with learning rate $10^{-3}$. The batch size is 2. A multi-step learning rate schedule is used with milestones at epochs 60 and 90 and decay factor 0.5. Validation is performed every 10 epochs. The checkpoint with the best validation loss is selected, and the corresponding test loss is reported.



\subsection{Implementation Details for the Inverse Darcy Experiment}
  \label{app:inverse_details}

  \paragraph{Problem setup.}
  We consider the inverse Darcy flow problem: given a fully observed pressure
  field $u^{\text{obs}} \in \mathbb{R}^{S \times S}$, recover the permeability
  field $a(x)$ satisfying $-\nabla\cdot(a\nabla u) = 1$ with homogeneous
  Dirichlet boundary conditions.  
  
  \paragraph{Forward model.}
  The forward MeanFlow model is trained on the forward Darcy task as described in Section~\ref{subsec:darcy_flow}. Its weights are fully frozen during     
  inverse optimization; no fine-tuning is performed. 
  \paragraph{Initial field construction.}             
  Rather than initializing the input state $z_0$ to the zero field, we
  construct a \emph{harmonic extension} of the boundary values of                  
  $u^{\text{obs}}$. Concretely, we fix the boundary pixels of $z_0$ to match
$u^{\text{obs}}|_{\partial\Omega}$ and iteratively average interior values       over 100 Jacobi iterations:                  \begin{equation}        
    z_0^{(k+1)}(x) = \tfrac{1}{4}\!\sum_{x' \sim x} z_0^{(k)}(x'),   
    \quad z_0|_{\partial\Omega} = u^{\text{obs}}|_{\partial\Omega}.                   
  \end{equation}        
  This initialization satisfies the boundary condition and provides a    spatially smooth starting point, which we find stabilizes early optimization steps.                     
  \paragraph{Hyperparameters.}
  We run Adam for 2000 steps with initial learning rate $5\times10^{-3}$,
  decayed to $10^{-4}$ via cosine annealing. Gradient norms are clipped to 1.0. All experiments use a single NVIDIA GPU. 
       

\subsection{Compared Methods}
\label{app:baselines}

We compare our method with baselines from three representative paradigms: physics-informed solvers, neural operator solvers, and generative PDE solvers.

\paragraph{PINN.}
Physics-informed neural networks directly parameterize the solution function with a neural network and train it by minimizing initial-condition, boundary-condition, and PDE residual losses. PINNs do not require paired simulation data, but their optimization can become unstable when different loss terms have conflicting scales or gradients.

\paragraph{I-PINN.}
I-PINN improves the vanilla PINN framework by combining an enhanced feed-forward architecture with adaptive loss weighting under upper-bound constraints. This design aims to mitigate gradient-flow stiffness and reduce gradient-related training failures. Compared with PINN, I-PINN improves convergence and accuracy without introducing substantial extra computational complexity.

\paragraph{FNO.}
Fourier Neural Operator is a data-driven neural operator that learns mappings between function spaces. It uses Fourier-domain integral operator layers to model global dependencies and can be evaluated on different spatial discretizations. In our experiments, FNO serves as a standard data-driven operator-learning baseline.

\paragraph{DeepONet.}
DeepONet is a neural operator architecture designed to learn nonlinear operators. It uses a branch network to encode the input function sampled at sensor locations and a trunk network to encode output coordinates. The two representations are combined to predict the output function value. Therefore, DeepONet is categorized as an operator-learning method rather than a standard pointwise neural network.

\paragraph{PINO.}
Physics-informed Neural Operator combines supervised operator learning with PDE residual constraints. PINO is built on the FNO framework and imposes physics losses, potentially at higher resolutions than the supervised data. It is a hybrid baseline between data-driven operator learning and physics-informed training.

\paragraph{CFO.}
Continuous Flow Operator learns continuous-time PDE dynamics using a flow-matching objective. Instead of autoregressively predicting future states on a fixed time grid, CFO fits temporal splines to trajectory data and trains a neural operator to predict analytic velocity fields. During inference, the learned continuous vector field is integrated with an ODE solver. We classify CFO as a continuous-time neural operator baseline.

\paragraph{DiffusionPDE.}
DiffusionPDE formulates PDE solving as generative modeling under partial observation. It models the joint distribution of solution and coefficient fields, allowing it to infer missing information and solve forward or inverse PDE problems. This makes it a representative diffusion-based generative PDE solver.

\paragraph{Rectified Flow.}
Rectified Flow learns a conditional velocity field that transports input laws to output solution laws along nearly straight trajectories. Compared with diffusion models that often require many stochastic denoising steps, rectified-flow surrogates use deterministic ODE sampling and can reduce the number of inference steps substantially.

\paragraph{PBFM.}
Physics-Based Flow Matching is a physics-constrained generative flow-matching method. It addresses the conflict between distributional fidelity and physical consistency by enforcing physical constraints during training with conflict-free gradient updates. Unlike methods that rely on inference-time correction, PBFM incorporates physics constraints into the generative training process.

\subsection{Additional Experimental Results of Sec. \ref{subsec:forward}}
\label{app:results}

\paragraph{Additional Visual Comparisons.}
In this section, we provide further qualitative comparisons between our proposed framework and baselines on NS dataset. Fig.~\ref{fig:visual_comparison} visualizes the ground truth and the predicted solutions. As illustrated, our method consistently yields predictions that are most aligned with the ground truth, especially in capturing fine-grained structures and high-frequency features. For instance, in the Navier--Stokes equation, while some baseline models struggle with the turbulent regions over long-term integration, our model maintains high physical fidelity. 

\begin{figure*}[htbp]
    \centering
\includegraphics[width=\textwidth]{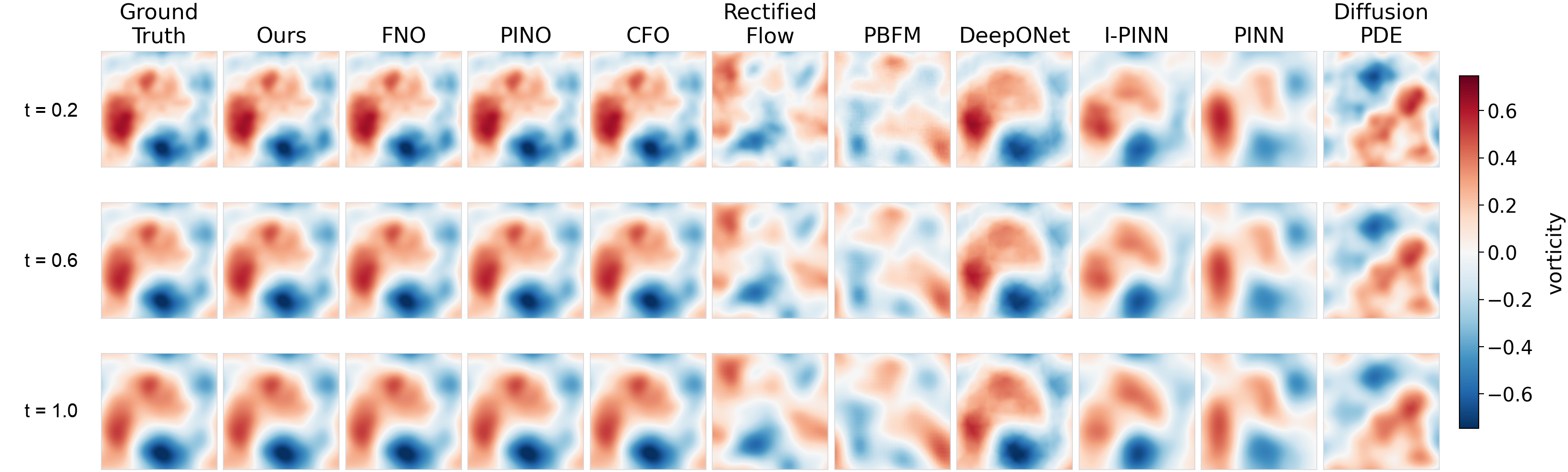}
    \caption{Visual comparison of predicted fields and error maps on ns dataset.}
    \label{fig:visual_comparison}
\end{figure*}

\paragraph{Training Loss Trajectories.}
We present the training trajectories in Fig.~\ref{fig:loss_curve}. The plots depict the evolution of the total loss, the data-driven loss ($\mathcal{L}_{\text{data}}$), and the proposed temporal and spatial MeanFlow losses ($\mathcal{L}_{t\text{-MF}}, \mathcal{L}_{s\text{-MF}}$). The results indicate that the data loss decreases rapidly during the initial epochs, allowing the model to capture the primary mapping between input conditions and solution fields. Subsequently, the temporal and spatial MeanFlow losses stabilize at low magnitudes, providing crucial physical regularizations that ensure the predicted fields satisfy the underlying PDE constraints. 

\begin{figure}[htbp]
    \centering
    \includegraphics[width=0.99\linewidth]{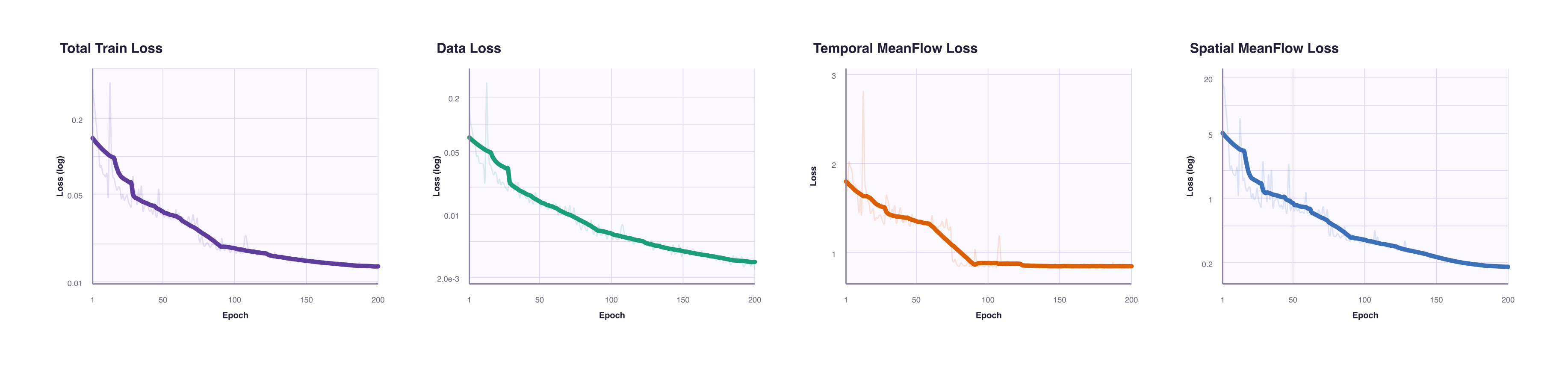} 
    \caption{Training trajectories for all loss components, including the total loss, the data loss, the temporal MeanFlow loss, and the spatial MeanFlow loss.}
    \label{fig:loss_curve}
\end{figure}

\paragraph{Additional Coefficient Analysis.}
We further analyze the sensitivity of Burgers' equation to different temporal and spatial MeanFlow loss coefficients. 
The coefficient pair $(\lambda_{\mathrm{T\mbox{-}MF}}, \lambda_{\mathrm{S\mbox{-}MF}})$ denotes the weights of the temporal MeanFlow loss and the spatial MeanFlow loss, respectively. 
All models are evaluated on the same base-resolution test set at $s=128$, and we report the global relative $\ell_2$ error.

As shown in Fig.~\ref{fig:burgers_coefficients}, introducing the spatial MeanFlow loss substantially improves performance when the temporal MeanFlow coefficient is fixed at $0.01$. 
In particular, increasing $\lambda_{\mathrm{S\mbox{-}MF}}$ from $0$ to a moderate value significantly reduces the error, while overly large coefficients do not further improve the result. 
Similarly, when fixing $\lambda_{\mathrm{S\mbox{-}MF}}=0.01$, using a moderate temporal MeanFlow coefficient gives the best performance, whereas a large temporal coefficient degrades accuracy. 
These results indicate that the temporal and spatial MeanFlow losses are complementary, but their coefficients should be balanced to avoid over-regularization.

\begin{figure}[t]
\centering
\includegraphics[width=\linewidth]{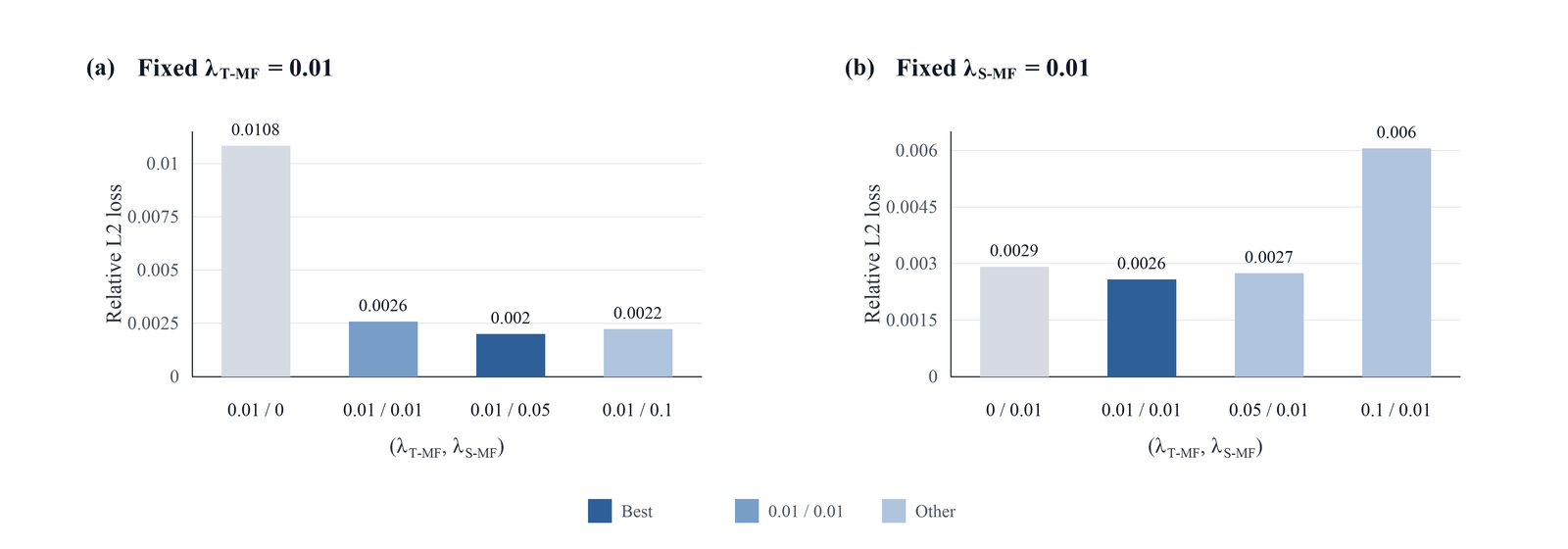}
\caption{Additional coefficient analysis on Burgers' equation. Left: spatial MeanFlow coefficient sweep with fixed temporal MeanFlow coefficient $\lambda_{\mathrm{T\mbox{-}MF}}=0.01$. Right: temporal MeanFlow coefficient sweep with fixed spatial MeanFlow coefficient $\lambda_{\mathrm{S\mbox{-}MF}}=0.01$. The reported metric is global relative $\ell_2$ error at the base resolution $s=128$; lower is better.}
\label{fig:burgers_coefficients}
\end{figure}




\end{document}